\documentclass[twoside]{article}

\usepackage[accepted]{aistats2026}
\usepackage{booktabs}
\renewcommand{\mathbf}{\boldsymbol}

\usepackage{url}
\usepackage[hidelinks]{hyperref}
\usepackage{amsmath,amsthm,amssymb,amsbsy, mathrsfs}
\usepackage{paralist}
\usepackage{xcolor}
\usepackage{color}
\usepackage{graphicx}
\graphicspath{{./figs/}}
\usepackage{algorithm}
\usepackage{algorithmic}
\usepackage{multirow}

\usepackage{enumerate}
\usepackage{enumitem}
\usepackage{graphicx}
\usepackage{fancyhdr}
\usepackage{cite}
\usepackage{cleveref}
\usepackage{subcaption}
\usepackage{bbm}
\usepackage{wrapfig}

\newtheorem{lemma}{Lemma}%
\newtheorem{cor}{Corollary}%
\newtheorem{prop}{Proposition}%

\theoremstyle{remark}

\theoremstyle{problem}

\definecolor{mplBlue}{RGB}{31, 119, 180}
\usepackage{tikz}

\newcommand{\R}{\mathbb{R}}

\newcommand{\e}{\begin{equation}}
\newcommand{\ee}{\end{equation}}
\newcommand{\en}{\begin{equation*}}
\newcommand{\een}{\end{equation*}}
\newcommand{\eqn}{\begin{eqnarray}}
\newcommand{\eeqn}{\end{eqnarray}}
\newcommand{\bmat}{\begin{bmatrix}}
\newcommand{\emat}{\end{bmatrix}}

\DeclareMathAlphabet\mathbfcal{OMS}{cmsy}{b}{n}

\newcommand{\mb}{\mathbf}
\newcommand{\mc}{\mathcal}

\newcommand{\vct}[1]{\boldsymbol{#1}}
\newcommand{\mtx}[1]{\boldsymbol{#1}}

\DeclareMathOperator*{\argmin}{\text{arg~min}}

\newcommand{\calD}{\mathcal{D}}

\newcommand{\calL}{\mathcal{L}}

\newcommand{\NC}{$\mathcal{NC}$}

\newcommand{\vb}{\vct{b}}

\newcommand{\vh}{\vct{h}}

\newcommand{\vw}{\vct{w}}
\newcommand{\vx}{\vct{x}}
\newcommand{\vy}{\vct{y}}

\newcommand{\vtheta}{\vct{\theta}}

\newcommand{\vmu}{\vct{\mu}}

\newcommand{\mW}{\mtx{W}}

\newcommand{\mTheta}{\mtx{\Theta}}

\graphicspath{{./figs/}}

\newlength{\imgwidth}
\newboolean{twoColVersion}
\setboolean{twoColVersion}{false}
\newcommand{\twoCol}[2]{\ifthenelse{\boolean{twoColVersion}} {#1} {#2} }

\newcommand{\forget}{\operatorname{f}}
\newcommand{\retain}{\text{r}}

\usepackage[dvipsnames]{xcolor}

\usepackage[markup=underlined]{changes}

\definechangesauthor[color=BrickRed]{JLM}
\definechangesauthor[color=NavyBlue]{YC}

\usepackage[round]{natbib}

\usepackage{titlesec}

\titlespacing*{\section}
{0pt}      %
{3pt}      %
{-1pt}      %

\titlespacing*{\subsection}
{0pt}      %
{1pt}      %
{-3pt}      %

\renewcommand{\bibname}{References}
\renewcommand{\bibsection}{\subsubsection*{\bibname}}

\begin{document}

\twocolumn[

\aistatstitle{An Illusion of Unlearning? Assessing Machine Unlearning Through Internal Representations}

\aistatsauthor{
  Yichen Gao$^{1,*}$ \And
  Altay Unal$^{2,*}$ \And
  Akshay Rangamani$^{2,\dagger}$ \And
  Zhihui Zhu$^{1,\dagger}$
}

\aistatsaddress{
$^1$Department of Computer Science \& Engineering, The Ohio State University\\
$^2$Department of Data Science,
New Jersey Institute of Technology\\
$^{*}$Equal contribution $^{\dagger}$Equal advising \\
}
 ]

\begin{abstract}

While numerous machine unlearning (MU) methods have recently been developed with promising results in erasing the influence of forgotten data, classes, or concepts, they are also highly vulnerable—for example, simple fine-tuning can inadvertently reintroduce erased concepts. In this paper, we address this contradiction by examining the {\it internal representations} of unlearned models, in contrast to prior work that focuses primarily on output-level behavior.
Our analysis shows that many state-of-the-art MU methods appear successful mainly due to a misalignment between last-layer features and the classifier—a phenomenon we call {\it feature–classifier misalignment}. In fact, hidden features remain highly discriminative, and simple linear probing can recover near-original accuracy. Assuming neural collapse in the original model, we further demonstrate that adjusting only the classifier can achieve negligible forget accuracy while preserving retain accuracy, and we corroborate this with experiments using classifier-only fine-tuning. Motivated by these findings, we propose MU methods based on a class-mean features (CMF) classifier, which explicitly enforces alignment between features and classifiers. Experiments on standard benchmarks show that CMF-based unlearning reduces forgotten information in representations while maintaining high retain accuracy, highlighting the need for faithful representation-level evaluation of MU. %

\end{abstract}

\section{Introduction}

\begin{figure*}[t]
    \centering
    \begin{subfigure}[t]{0.95\textwidth}
        \centering
        \includegraphics[width=\linewidth]{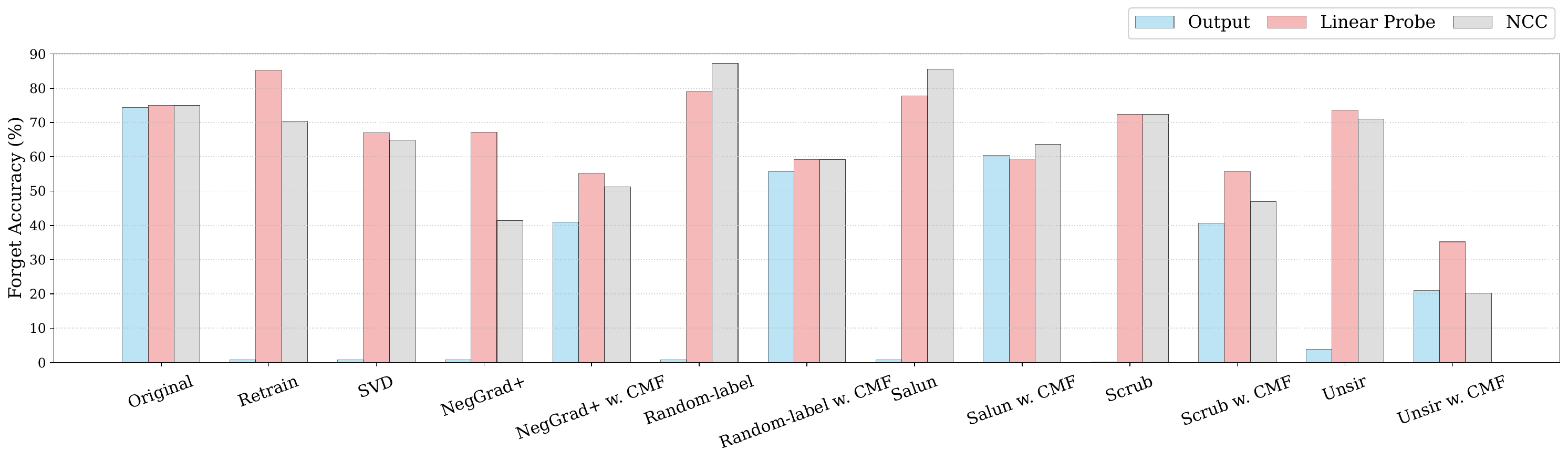}
    \end{subfigure}
\vspace{-.1in}
    \caption{Comparison of forget accuracies evaluated in the output level (blue bars) VS feature-level via linear probe (red bars) and nearest class center accuracy (NCC, grey bars) for original, retain-only retrain, and various MU methods and those with CMF classifier on cifar100 with forget 1 class scenario.
   }
    \label{fig:linear_probe_bars}
\vspace{-.15in}
\end{figure*}

\begin{figure}[t]
        \centering
        \includegraphics[width=\linewidth]{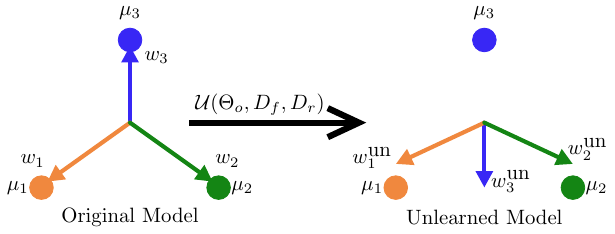}%
        \vspace{-.1in}
        \caption{Visualization of Proposition \ref{prop:nc_unlearn_inf}. We observe a misalignment caused by the unlearning methods between the class mean and the corresponding classifier weights for the {\color{blue}forget class (blue)} while the alignment is mainly preserved for the retain classes ({\color{ForestGreen}green} and {\color{orange}orange}). %
        The spheres represent the class means while the arrows represent the last layer classifier weights. %
        } %
        \vspace{-.25in}
        \label{fig:prop_1_vis}
 \end{figure}

Machine Unlearning (MU) \citep{bourtoule2021machine} aims to remove the influence of specific training samples from a model without retraining it from scratch. This capability is increasingly critical in practice due to requirements such as compliance with privacy regulations, protection of intellectual property and copyrighted content, and safety concerns arising from the retention of harmful or biased data \citep{voigt2017eu}. Beyond performance, unlearning directly impacts the trustworthiness and ethical deployment of machine learning systems \citep{jin2023forgettable} since some of the data might be tainted \citep{jagielski2018manipulating} or the data might contain harmful biases \citep{fabbrizzi2022survey}.

Due to these practical and ethical demands, MU has been extensively studied and empirically demonstrated across a range of tasks, including in classifiers \citep{choi2023towards, Golatkar_2020_CVPR, tarun2023fast} and generative models \citep{fan2023salun, li2024machine}. MU in classification focuses on forgetting individual examples or entire classes used in training, with the goal of erasing their influence while preserving performance on the remaining data. MU in generative models targets the removal of specific concepts, ensuring that the model cannot produce outputs based on them. 

Despite these encouraging results, recent studies have also revealed significant vulnerabilities of MU. For example, MU can be unstable in generative models, where forgotten concepts may re-emerge. In particular, simply fine-tuning on seemingly unrelated images can inadvertently reintroduce erased content \citep{suriyakumar2024unstable}. In addition, unintended concepts can also be affected by the erase process of a concept \citep{lu2024mace, yu2025forgetme}.  This raises a fundamental question: {\it do unlearned models truly forget, and how should we faithfully assess their performance?}

To investigate these questions, 
we focus on forgetting entire classes in the context of image classification, since the presence of label information makes it easier to assess the effectiveness of unlearning through accuracy metrics. Classification has 
served as a testbed for developing and validating MU methods, and 
heuristic approaches have been proposed in recent years \citep{choi2023towards, kurmanji2023towards, tarun2023fast, kodge2024deep}. Yet, measuring the effectiveness of MU remains a challenging problem. Current evaluations primarily rely on output-level metrics, which measure model predictions on the forget set ({\it forget accuracy}) and the retain set ({\it retain accuracy}). However, it remains unclear whether forgetting truly occurs at the level of internal feature representations or whether 
MU methods merely suppress classifier outputs while forgotten concepts persist in the representation space. In this work, we propose to assess unlearning effectiveness by studying the internal representations rather than relying solely on output-level metrics.

 \begin{figure}[t]
\centering
\begin{subfigure}[t]{0.49\linewidth}
  \centering
  \includegraphics[width=\linewidth]
  {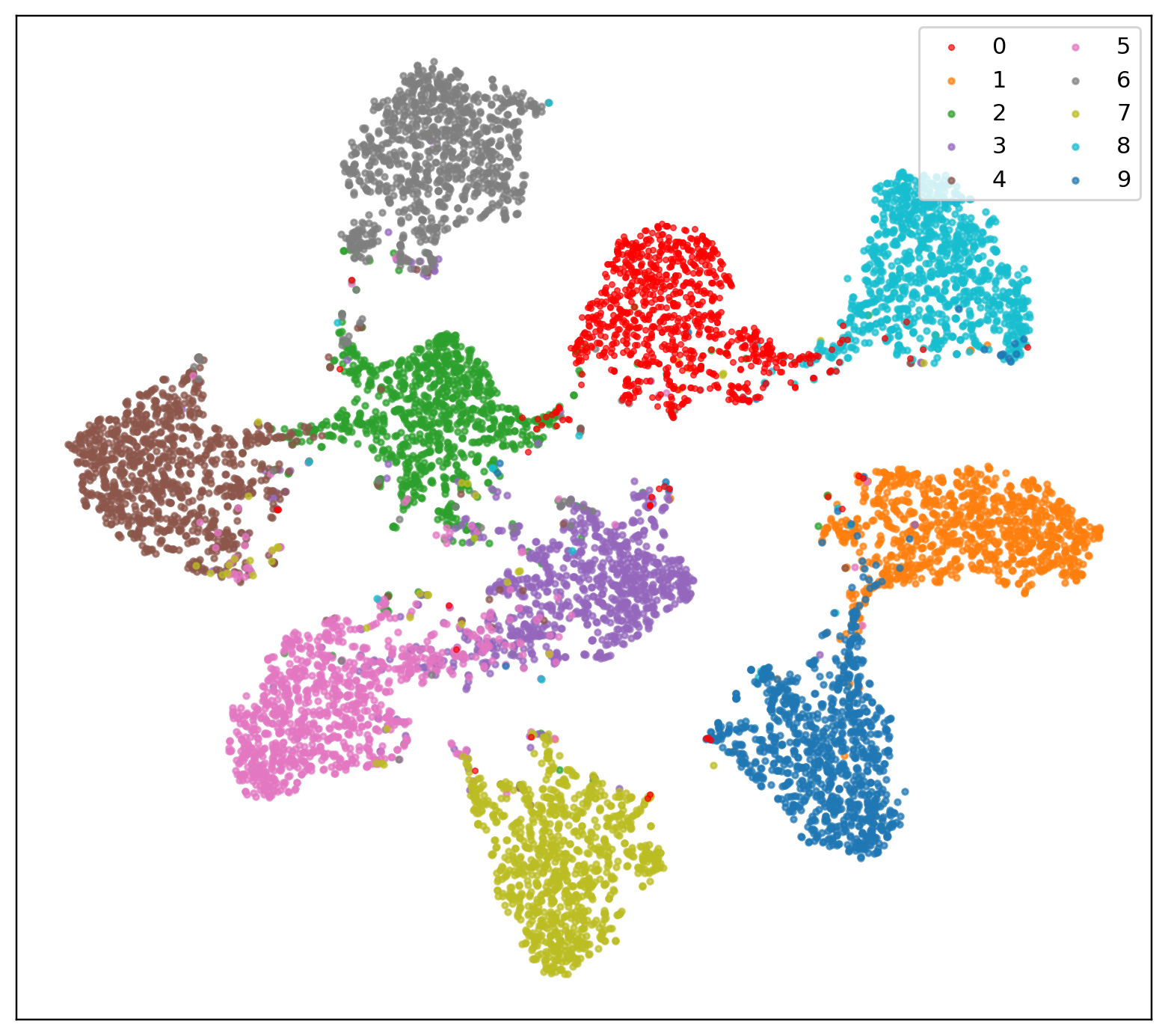}
  \caption{Original}
\end{subfigure}\hfill
\begin{subfigure}[t]{0.49\linewidth}
  \centering
  \includegraphics[width=\linewidth]{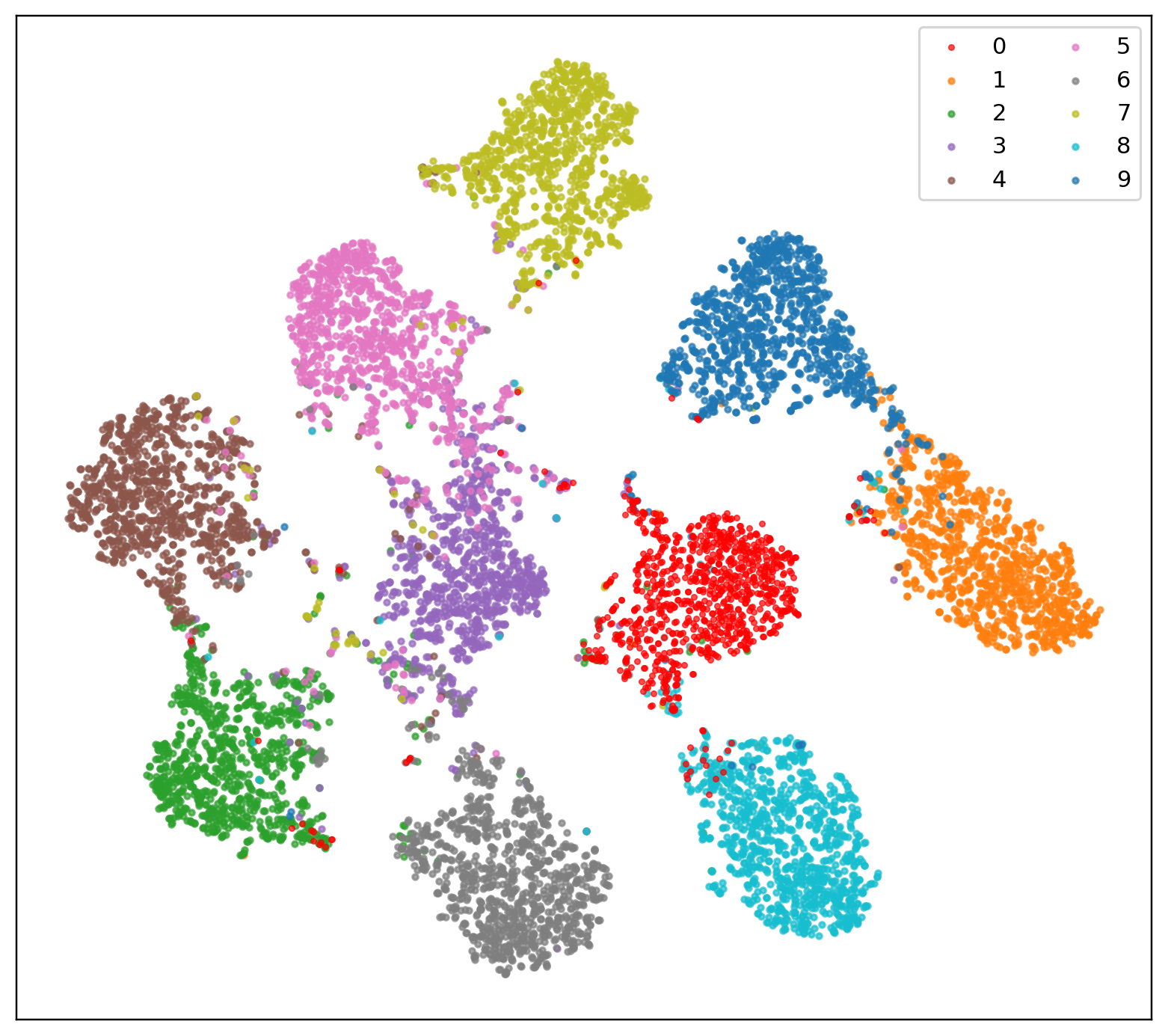}
  \caption{Random-label unlearned}
\end{subfigure}\hfill

\vspace{-.1in}
\caption{t-SNE on CIFAR-10 with (a) original model and (b) Random-label unlearned model.
The forgotten class (dark blue points) remains linearly separable in unlearned models.}
\label{fig:tsne_cifar10_baseline}
\vspace{-.2in}
\end{figure}

\vspace{-.1in}
\paragraph{Contribution} Our main contributions are summarized as follows:
\begin{itemize}
[leftmargin=*]
    \item We find that while many state-of-the-art MU methods, including Random Label, SalUn, NegGrad+, SCRUB, and UNSIR, achieve negligible forget accuracy, their hidden-layer features remain highly discriminative; see \Cref{fig:tsne_cifar10_baseline} for t-SNE plot. As shown in \Cref{fig:linear_probe_bars}, a simple linear probe on the last-layer representation can recover near-original accuracy. This reveals that even when unlearning appears successful according to standard metrics (forget and retain accuracy), current methods often fail to remove information from the hidden representation space, leaving latent traces of forgotten data that can be recovered with retraining.

    \item Inspired by the {\it neural collapse (NC)} phenomenon in deep classifiers \citep{papyan2020prevalence} (see Section~\ref{sec:background_DNN and NC} for details), we study the alignment between class-mean features and the classifier. We show that after unlearning, self-duality between the classifier and last-layer class-mean features persists for retain classes—classifiers remain almost perfectly matched with their class means—but for forget classes there exists significant {\it feature–classifier misalignment} as depicted in \Cref{fig:prop_1_vis}. Assuming NC in the original model, we further demonstrate that a simple MU method can be constructed by adjusting the classifier, yielding negligible accuracy on the forget set while preserving accuracy on the retain set. We corroborate this analysis by applying SOTA MU methods with classifier-only fine-tuning, which achieve comparable 
    forgetting and retaining accuracy at the output level, underscoring the limitations of current evaluation metrics. 

    \vspace{-.05in}
    
    \item Motivated by our analysis, we propose a representation-level unlearning framework that enforces alignment between features and the classifier, ensuring that forgetting occurs within the hidden representations as well. Specifically, we employ a class-mean features (CMF) classifier, which explicitly sets each classifier weight to the mean feature vector of its corresponding class and can be seamlessly integrated into existing MU methods. As shown in \Cref{fig:linear_probe_bars}, experiments on standard benchmarks demonstrate that CMF-based unlearning substantially reduces the retention of forgotten information at the representation level (i.e., achieving much lower forget accuracy under linear probing), while maintaining high accuracy on retained data.

\end{itemize}
\vspace{-.2in}
We make our code publicly available at \url{https://github.com/ycgao1/CMF_Unlearning}.

\section{Preliminaries, Machine Unlearning, and Its Evaluation}

\subsection{Neural Networks and Neural Collapse}
\label{sec:background_DNN and NC}
A standard deep neural network (DNN) classifier $f:\R^{d_{\text{in}}}\rightarrow \R^K$ consists of a multi-layer nonlinear compositional
feature mapping $\phi_{\vtheta}:\R^{d_{\text{in}}} \rightarrow \R^d$ with $\vtheta$ denoting the network parameters in the feature mapping and a linear classifier $(\mW,\vb)$ with $\mW = [\vw_1, \vw_2, \dots, \vw_K]^\top \in \mathbb{R}^{K \times d}$ and $\vb$, which can  expressed as 
\setlength{\belowdisplayskip}{2pt} \setlength{\belowdisplayshortskip}{2pt}
\setlength{\abovedisplayskip}{2pt} \setlength{\abovedisplayshortskip}{2pt}
\e
    f_{\mTheta}(\vx) = \mW \phi_{\vtheta}(\vx) + \vb \;\in\; \mathbb{R}^K.
\label{eq:DNN}\ee
Here $\mTheta = \{\vtheta,\mW,\vb\}$ denotes all the network parameters.
The feature extractor $\phi_{\vtheta}(\cdot)$ generates the data-dependent feature vectors in $\mathbb{R}^d$, 
while the linear classifier $(\mW,\vb)$ determines the linear decision boundary in the feature space. 

With an appropriate loss function, the parameters $\mTheta$ of the network are optimized to learn the underlying relation between an input sample $\vx$ and its corresponding target $\vy$, such that the network output $f_{\mTheta}(\vx)$ approximates $\vy$. Specifically, let $\mathcal{D}=\{(\vx_{i,k}, \vy_{i,k})\}_{i=1}^N$ be a dataset of $N$ training samples, where $\vx_{i,k}$ is the $i$-th sample from $k$-th class and $\vy_{i,k} \in \mathbb{R}^K$ is the corresponding one-hot label vector. The parameters $\mTheta$ are learned by minimizing the empirical risk over all the training samples:
\e
\mTheta_o = \argmin_{\mTheta} \sum_{k=1}^K\sum_{i=1}^{n_k} \calL(f_{\mTheta}(\vx_{i,k}),\vy_{i,k}),
\ee
where $\calL(f_{\mTheta}(\vx_{i,k}),\vy_{i,k})$ is a predefined loss function, such as the cross-entropy loss, that appropriately measures the discrepancy between
the output $f_{\mTheta}(\vx_{i,k})$ and the target $\vy_{i,k}$.

\paragraph{Neural Collapse}
Neural Collapse (\NC) \citep{papyan2020prevalence} is an intriguing phenomenon observed in the last-layer classifier and feature representations during the terminal phase of training (TPT), when the training error approaches zero. In this regime, features from the final layer align with their corresponding class mean vectors, which collectively form a simplex equiangular tight frame (ETF) structure.

More precisely, \NC comprises the following properties:
\textbf{(i) Variability collapse (\NC$_1$):} features within each class collapse to their class mean;
\textbf{(ii) Simplex ETF structure (\NC$_2$):} the class means, centered at their global mean, are not only linearly separable but are maximally separated and form a simplex ETF;
\textbf{(iii) Feature–classifier alignment (\NC$_3$):} each class mean is perfectly aligned with the corresponding last-layer linear classifier;
\textbf{(iv) Nearest class center decision rule (\NC$_4$):} the last-layer classifier becomes equivalent to a nearest class center (NCC) classifier.

To quantify \NC, let ${\vh}_{i,k} = \phi_{\vtheta}(\vx_{i,k})$ denote the learned feature representation of sample $\vx_{i,k}$ from class $k$. We define the class-wise mean features and the global mean feature as
\begin{equation}
\vmu_k := \frac{1}{N_k}\sum_{i=1}^{N_k} {\vh}_{i,k},
\qquad
\vmu_G := \frac{1}{K} \sum_{k=1}^{K} \vmu_k .
\label{eq:means}
\end{equation}

Neural collapse characterizes the convergence of features ${\vh}_{i,k}$ toward their corresponding class means $\vmu_k$, along with the alignment of the classifier weights $\vw_k$ with these means. 

In the context of unlearning, two NC measures are particularly informative.
The first is \textbf{feature–classifier alignment}, measured by
\begin{align}
{\mc NC}_3 := \left\| \frac{\vw_k}{\|\vw_k\|} - \frac{\vmu_k - \vmu_G}{\|\vmu_k - \vmu_G\|} \right\|,
\label{eq:NC3}
\end{align}
which quantifies the alignment between the normalized classifier weight $\vw_k$ and the centered class mean $\vmu_k - \vmu_G$.

The second measure is the \textbf{nearest class center (NCC) classification accuracy}, defined as
\begin{align}
\text{NCC}:= \mathbb{P}\left[ y = \arg\min_k \left\| \phi_{\vtheta}(\vx) - \vmu_k \right\|_2 \right],
\end{align}
where $\phi_{\vtheta}(\vx)$ denotes the  representation of input $\vx$, and the probability is taken over data samples $(\vx, y)$. In words, under the NCC rule, a sample $\vx$ is assigned to the closest class mean.

In this paper, we adopt \NC\ analysis as a diagnostic tool for unlearning by tracking the \NC$_3$ and NCC metrics throughout the unlearning process.

\subsection{Machine Unlearning}
\label{sec:MU}
Machine unlearning (MU) is a paradigm that aims to make a machine learning (ML)  model forget about certain data. Originally motivated by privacy concerns and the ``right to be forgotten'', the goal of machine unlearning is to allow people to opt out of their data being used in the training of ML models. Machine unlearning is also useful in contexts outside of privacy such as correcting models trained on erroneous data \citep{ali2025evaluating}, removing classes from classifier, etc.
Since training ML models from scratch may be quite expensive, machine unlearning aims to provide a sustainable solution for such cases. 

Given a dataset $\calD$, let $\mathcal{D}_{f} \subset \mathcal{D}$ denote the subset of data targeted for unlearning, referred to as the forget set. Its complement, $\mathcal{D}_{\text{\retain}} = \mathcal{D} \setminus \mathcal{D}_{\forget}$, is the portion of the dataset to be retained, referred to as the retain set. The content of forget and retain sets varies according to the application. For the class unlearning scenario, $\mathcal{D}_{\forget}$  and $\mathcal{D}_{\retain}$ denote data corresponding to the forgotten classes $C_{\forget}$ and retained classes $C_{\retain}$, respectively. $\mathcal{D}_{\forget}$ contains all the examples belonging to $C_{\forget}$ while $\mathcal{D}_{\retain}$ contains the rest of the training data. 

In the literature, retraining a fresh model $\mTheta_{\text{r}}$ solely on $\mathcal{D}_{\retain}$ is widely regarded as the \emph{gold standard} for MU \citep{bourtoule2021machine,279996}. 
Nevertheless, full retraining is both computationally expensive and time-consuming, which is impractical for large-scale models or frequent removal requests. Recent research therefore focuses on designing approximate methods that modify the original model $\mTheta_{\text{o}}$ to achieve the effect of unlearning. Formally, given training data $\calD$ and an original trained model $\mTheta_{\text{o}}$, an unlearning algorithm defines a transformation
\begin{equation}
    \mTheta_{\text{u}} \;=\; \mathcal{U}(\mTheta_{\text{o}},\,\mathcal{D}_{\forget},\,\mathcal{D}_{\retain}),
    \label{eq:general_mu}
\end{equation}
where $\mTheta_{\text{u}}$ is the unlearned model and $\mathcal{U}$ denotes the unlearning operator.

NegGrad \citep{Golatkar_2020_CVPR,choi2023towards} performs gradient ascent on the forget set, sometimes combined with a retain loss to mitigate over-forgetting.  
Random-label \citep{Golatkar_2020_CVPR} assigns random labels to the forget samples, forcing the model to fit noise and degrade its predictive ability on $\mathcal{D}_{\forget}$.  
Saliency Unlearning (SalUn) \citep{fan2023salun} improves upon this by updating only parameters most salient to the forget set, enhancing efficiency and stability.  
SCRUB \citep{kurmanji2023towards}
formulates unlearning as a selective knowledge distillation problem,
encouraging the model to diverge from the teacher on the forget set
while preserving behavior on the retain set.
UNSIR \citep{tarun2023fast} generates error-maximizing noise to impair model weights associated with the forget classes, followed by a repair step using retain data to restore overall model performance.
Beyond gradient-based strategies, SVD-based unlearning \citep{kodge2024deep} offers a gradient-free alternative by projecting feature representations onto the orthogonal complement of the forget subspace to suppress discriminative information.

\section{Evaluation of Unlearning}

Evaluating the effectiveness of machine unlearning is challenging. Currently, machine unlearning methods are mainly evaluated using output-level metrics, which focus solely on model predictions. These metrics are unsuitable for assessing unlearning in the learned representation space \citep{xu2024don} since the inner representation space has larger dimensionality. In addition to output-level metrics, some models consider relearn time as a metric for evaluating machine unlearning \citep{xue2025towards}, which refers to the number of epochs for an unlearned model to relearn and restore its performance on the forgotten data. However, this is also unsuitable, since we will explain further that performance on forgotten data can be easily retrieved.

In this section, we first describe the current output level evaluation metrics, then propose feature-level evaluation metrics for machine unlearning. In this section, we first review existing output-level evaluation metrics, and then introduce feature-level evaluation metrics for machine unlearning. Most prior work evaluates the effectiveness of unlearning by measuring the performance of the entire network at the output layer. From this perspective, we refer to such evaluations as shallow unlearning. In contrast, we also assess the effectiveness of unlearning at the feature level, which we term deep unlearning.

\begingroup
\setlength{\tabcolsep}{4pt}
\small

\begin{table*}[ht]
\centering
\caption{
Evaluation of various MU methods on three datasets for unlearning certain number of classes. For all the (unlearned) models, we report both mean forget accuracy and mean retain accuracy evaluated for the entire model (labeled as Output) and the feature mapping by linear probe and NCC classification accuracy.
}
\vspace{-.15in}
\label{tab:ce_linearprobe}
\resizebox{1.0\linewidth}{!}{
\begin{tabular}{l l cc cc cc cc cc cc}
\toprule
\multirow{3}{*}{\textbf{Method}}
  & \multirow{3}{*}{\textbf{Accuracy}}
  & \multicolumn{4}{c}{\textbf{CIFAR-10}}
  & \multicolumn{4}{c}{\textbf{CIFAR-100}}
  & \multicolumn{4}{c}{\textbf{Tiny-ImageNet}}\\
\cmidrule(lr){3-6}\cmidrule(lr){7-10}\cmidrule(lr){11-14}
  &  & \multicolumn{2}{c}{1} & \multicolumn{2}{c}{3}
     & \multicolumn{2}{c}{1} & \multicolumn{2}{c}{10}
     & \multicolumn{2}{c}{1} & \multicolumn{2}{c}{20}\\
\cmidrule(lr){3-4}\cmidrule(lr){5-6}\cmidrule(lr){7-8}\cmidrule(lr){9-10}\cmidrule(lr){11-12}\cmidrule(lr){13-14}
  &  & Retain & Forget & Retain & Forget & Retain & Forget & Retain & Forget & Retain & Forget & Retain & Forget \\
\midrule

\multirow{3}{*}{Original}
   & Output & 93.98 & 93.98 & 94.00 & 93.94 & 74.61 & 74.40 & 74.47 & 75.88  & 65.27 & 58.80 & 65.15 & 66.02 \\
 & Linear Probe & 94.02 & 94.02 & 94.03 & 94.00 & 74.53 & 75.00 & 74.38 & 75.90 & 65.10 & 60.80 & 64.97 & 66.08 \\
 & NCC & 94.00 & 93.99 & 94.03 & 93.92 & 74.40 & 75.00 & 74.28 & 75.68 & 64.65 & 60.00 & 64.56 & 65.00 \\
\midrule
\multirow{3}{*}{\shortstack{Retain-only\\Retrain}}
  & Output & 94.74 & 0.00 & 95.37 & 0.00 & 76.01 & 0.00 & 76.50 & 0.00 & 66.52 & 0.00 & 66.38 & 0.00 \\
 & Linear Probe & 90.49 & 77.35 & 85.64 & 67.33 & 74.09 & 85.20 & 69.34 & 60.94 & 65.90 & 46.40 & 65.21 & 30.36 \\
 & NCC & 93.31 & 47.06 & 91.37 & 37.07 & 73.90 & 70.40 & 70.98 & 43.18 & 63.57 & 71.20 & 59.37 & 44.30 \\

\midrule

\multirow{3}{*}{\shortstack{Retain-only\\FT}}
  & Output         & 94.26 & 47.67 & 95.24 & 52.48 & 74.08 & 53.20 & 74.53 & 64.96 & 65.26 & 37.60 & 65.60 & 50.92 \\
  & Linear Probe   & 93.94 & 89.71 & 94.14 & 90.44 & 73.89 & 73.80 & 73.97 & 74.50 & 64.44 & 56.00 & 64.05 & 63.82 \\
  & NCC   & 93.70 & 89.63 & 93.80 & 88.75 & 74.04 & 73.20 & 74.08 & 73.56 & 64.00 & 57.60 & 63.90 & 63.58 \\
  
\midrule

\multirow{3}{*}{NegGrad+}
  & Output        & 92.85 &  0.00 & 93.29 & 0.01 & 69.90 & 0.00 & 70.80 & 0.28 &  57.96 & 0.00 & 59.06 & 0.00 \\
  & Linear Probe   & 92.14 & 67.09 & 88.18 & 73.91 & 72.55 & 67.20 & 72.20 & 62.32 & 60.43 & 58.00 & 60.75 & 54.68 \\
  & NCC   & 91.33 & 52.00 & 87.28 & 59.17 & 71.58 & 41.40 & 70.08 & 41.24 & 59.01 & 37.60 & 56.45 & 36.78 \\
  
\midrule

\multirow{3}{*}{SVD}
  & Output        & 92.02 &  0.00 & 94.05 & 57.43 &71.09 &  0.00 & 73.10 & 55.56 & 64.43 & 2.00 & 65.45 & 59.02 \\
  & Linear Probe   & 90.44 & 61.80 & 92.43 & 83.58 & 73.14 & 67.00 & 73.38 & 74.08 & 63.10 & 60.40 & 63.03 & 63.64 \\
  & NCC   & 90.11 & 34.54 & 93.23 & 72.32 & 71.81 & 64.80 & 73.33 & 72.22  & 64.65 & 60.00 & 64.56 & 65.00  \\

\midrule

\multirow{3}{*}{Random-label}
  & Output        & 92.93 &  0.00 & 94.14 &  0.00 & 72.33 & 0.00 & 72.19 & 0.00 & 65.48 & 0.40 & 64.85 & 0.98 \\
  & Linear Probe   & 92.65 & 92.49 & 92.45 & 90.25 & 73.08 & 79.00 & 72.08 & 72.08 & 64.07 & 58.00 & 62.02 & 57.82\\
  & NCC   & 92.25 & 80.25 & 91.92 & 73.22 & 72.38 & 87.20 & 70.67 & 62.22& 63.20 & 69.20 & 59.42 & 42.72 \\

\midrule

\multirow{3}{*}{SalUn}
  & Output       & 93.19 &  0.00 & 94.43 &  0.00 & 72.96 & 0.00 & 72.92 & 0.06 & 65.49 & 0.40 & 64.63 & 3.40 \\
  & Linear Probe   & 93.05 & 92.57 & 92.89 & 89.63 & 73.26 & 77.80 & 72.10 & 72.66 & 64.07 & 55.60 & 61.86 & 56.52  \\

  & NCC   & 91.31 & 93.70 & 91.99 & 68.45 & 72.65 & 85.60 & 70.62 & 63.34 & 63.33 & 68.80 & 59.07 & 39.70   \\
\midrule
\multirow{3}{*}{SCRUB}
  & Output        & 91.37 & 0.00 & 93.61 &  0.00 & 73.67 & 0.20 & 74.56 & 0.32 & 65.43 & 1.20 & 65.30 & 5.48  \\
  & Linear Probe   & 91.71 & 74.79 & 91.88 & 78.23 &  73.69 & 72.40 & 73.05 & 66.92 & 64.44 & 56.00 & 62.87 & 56.90 \\

  & NCC   & 89.96 & 55.30 & 89.73 & 47.52 & 73.51 & 72.40 & 72.35 & 53.58 & 64.37 & 60.40 & 61.41 & 53.50 \\

  \midrule
\multirow{3}{*}{UNSIR} & Output & 91.84 & 0.48 & 92.87 & 0.01 & 73.77 & 3.80 & 73.58 & 14.58 & 64.66 & 0.00 & 65.46 & 9.72 \\
 & Linear Probe & 89.71 & 85.29 & 88.21 & 70.59 & 73.40 & 73.60 & 72.15 & 67.64 & 63.88 & 61.20 & 62.87 & 61.10 \\
 & NCC & 89.73 & 62.72 & 87.73 & 49.42 & 73.05 & 71.00 & 71.26 & 59.24 & 63.16 & 59.20 & 61.80 & 56.14 \\
\bottomrule
\end{tabular}
}
\vspace{-.3in}
\end{table*}

\endgroup

\subsection{Evaluation Metrics for Shallow Unlearning}
Broadly, evaluation is based on two aspects: unlearning effectiveness, measured on the forget set, and post-unlearning model utility, measured on the retain set. Accuracy-based metrics are the most widely used. Specifically, given a testing datset $\calD$, {\it forget accuracy}, denoted by $\mathrm{Acc}_{\forget}$, quantifies prediction performance on the forget set $\mathcal{D}_{\forget}$, while retain accuracy, denoted by $\mathrm{Acc}_{\retain}$, measures performance on the retain set $\mathcal{D}_{\retain}$. An effective unlearning method should substantially reduce $\mathrm{Acc}_{\forget}$ while maintaining high $\mathrm{Acc}_{\retain}$ on test data. %

Another common evaluation metric is the success of membership inference attacks (MIA) \citep{shokri2017membership}, which tests whether an adversary can determine if a sample was included in training. MIA is a useful metric to measure privacy guarantees where individuals would like their data to be excluded. In this paper we focus on class forgetting in the context of classification, and evaluating whether ``forgotten'' knowledge can be retrieved.
In this context MIA is not a relevant metric \citep{kurmanji2023towards}.

\subsection{Evaluation Metrics for Deep Unlearning}
According to (\ref{eq:DNN}), a DNN classifier consists of two components: the feature mapping $\phi_{\vtheta}$ and the linear classifier. If the model performs poorly, the source of error may lie in either the feature mapping or the classifier. Similarly, in the context of unlearning, even if the classifier is updated to suppress performance on the forget set, the feature extractor may still preserve discriminative information about the forgotten data, which means that the applied unlearning method may not have succeeded at all. Output-level metrics, therefore, risk overlooking hidden representations that continue to encode forgotten concepts as they simplify the evaluation of the unlearning methods. A robust MU method is expected to remove the influence of the forget set across the entire model, particularly within the feature mapping.

Motivated by this, we propose to evaluate MU not only at the output level but also at the representation level. Specifically, we assess the effectiveness of the feature mapping in terms of its discriminative and predictive ability. A common tool for this purpose is the linear probe: a new linear classifier is trained on top of the frozen features using the full dataset $\mathcal{D}=\mathcal{D}_{r}  \cup \mathcal{D}_{f}$, after which performance on the forget set and retain set is evaluated. We also adopt \NC\ analysis as we mentioned in Section \ref{sec:background_DNN and NC}. \NC\ describes how features $\vh_{i,k}$ of each class converge to their class mean $\vmu_k$, and classifier weights $\vw_k$ align with these means. We will be utilizing \NC$_3$ and \NC$_3$ for the unlearning to evaluate the unlearned features.

Overall, feature-level metrics (linear probe, \NC$_3$, NCC) provide a more faithful evaluation of unlearning by directly testing whether forgotten concepts survive in the internal representation, complementing output-level metrics. With these metrics, we can actually observe the true performance of the unlearning methods as the evaluation is not restricted to the accuracy metrics anymore. Unlike output-level evaluation metrics, there has not been a universally established feature-level evaluation metric. Although there have been some recent developments in terms of defining a new feature-level metric regarding the representations based on information theory \citep{jeon2024information} and Centered Kernel Alignment (CKA) \citep{kim2025we}, both metrics compare against a model trained on only the retain set $\mathcal{D}_r$ which is a significant computational expense.

\section{The Illusion of Unlearning}

\paragraph{Data, models, and unlearning setups.}
Following prior work \citep{kodge2024deep,fan2023salun}, we conduct
class-unlearning experiments on both convolutional and transformer-based
architectures.
For convolutional models, we use ResNet architectures \citep{he2016deep},
evaluating ResNet-18 on CIFAR-10 and CIFAR-100 \citep{krizhevsky2009learning},
and ResNet-50 on Tiny-ImageNet \citep{le2015tiny}.
For each dataset, we consider both single-class and multi-class unlearning
scenarios.
Specifically, the numbers of forgotten classes $|\mathcal{D}_f|$ are
$\{1,3\}$ for CIFAR-10,
$\{1,10\}$ (corresponding to two super-classes) for CIFAR-100,
and $\{1,20\}$ for Tiny-ImageNet.

For ResNet-based experiments, we evaluate several types of models:
(i) the \emph{Original} model trained on the full dataset,
(ii) the \emph{Retain-only Retrain} model trained from scratch using only
the retain dataset,
(iii) the \emph{Retain-only Fine-tuning (FT)} model obtained by fine-tuning the original model on the retain dataset,
and (iv) various unlearned models produced by applying representative machine
unlearning methods on the original model.
In particular, we consider six unlearning algorithms:
Retain-only FT, NegGrad+, Random Label, SalUn, SVD, SCRUB, and UNSIR,
as described in \Cref{sec:MU} while experimental setup is described in Appendix \ref{unlearning_scenrios}.
We also report additional results with variance across ResNet and Vision Transformers (ViT-S/16) in Appendix~\ref{app:tables}.

\subsection{Shallow Unlearning with Persistent Discriminative Features}

\begin{figure}[t]
        \centering
        \includegraphics[width=\linewidth]{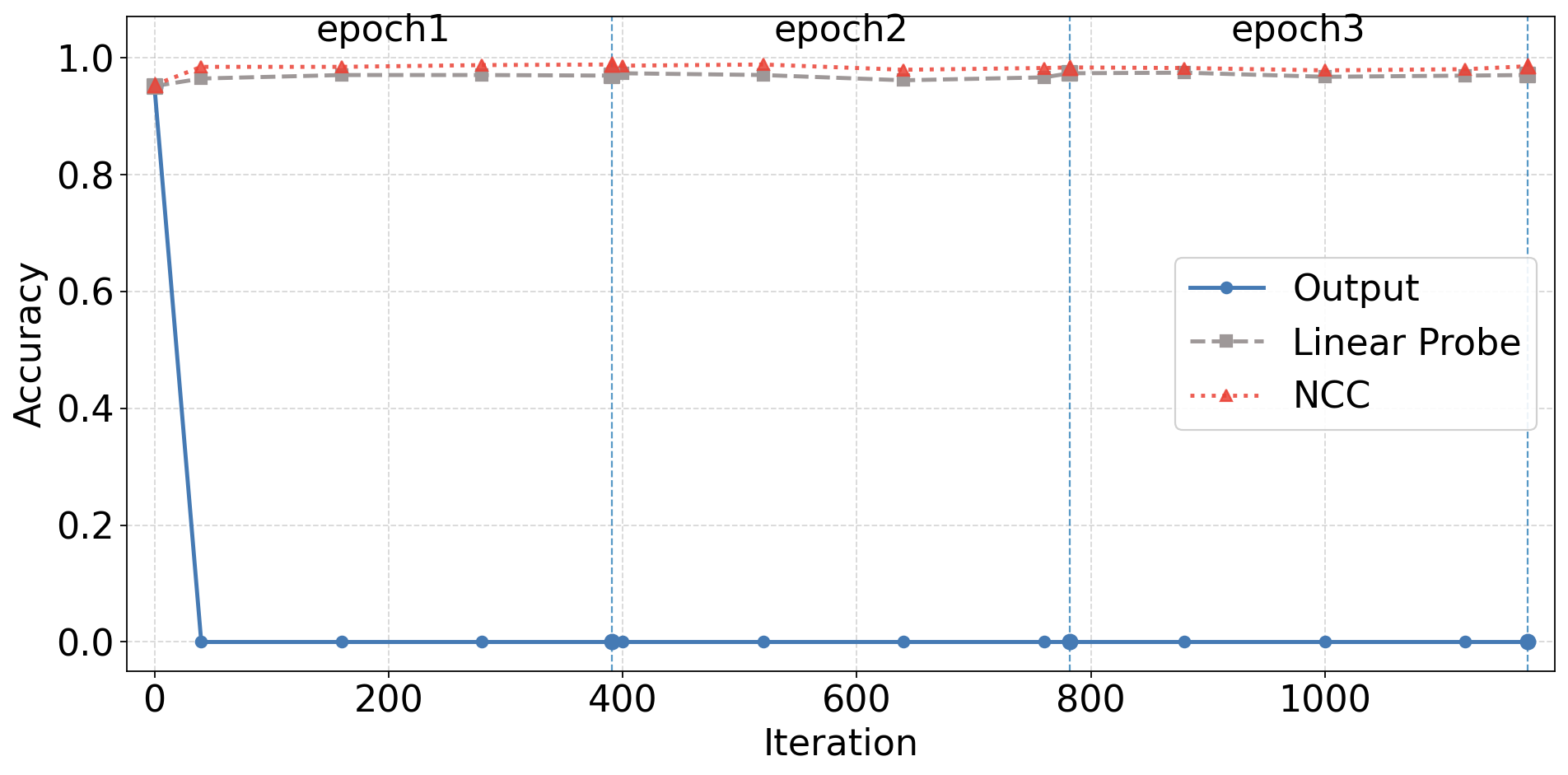}
       \vspace{-.1in} 
        \caption{
Learning curve of Random-label unlearning on CIFAR-10 when forgetting class 0 (airplane). While the output-level forget accuracy drops to zero quickly, the linear-probe and NCC accuracies remain consistently high throughout the unlearning process.
}
        
        \label{fig: random_label forget learning curve}
        \vspace{-.2in}
 \end{figure}

Table~\ref{tab:ce_linearprobe} displays the effectiveness of different MU methods in terms of mean output-level accuracies and mean feature-level accuracies (evaluated by linear probe and NCC) for both forget set (forget accuracy).

\vspace{-.1in}
\paragraph{Observation 1: Representations remain linearly separable after unlearning.} 
As shown in Table~\ref{tab:ce_linearprobe}, many unlearning algorithms appear successful when evaluated at the output level: the accuracy on the forget set drops to nearly zero, suggesting effective forgetting. 
However when the learned representations are frozen and a linear probe is trained on top of them, the forget accuracy recovers to a high level.
A similar trend is observed under the NCC accuracy indicating that the ``forgotten'' representations still cluster around their class means.
This implies that while current MU methods appear successful according to standard output-level metrics (forget and retain accuracy) they often fail to remove information from the hidden representation space, leaving latent traces of forgotten data that can be recovered by a simple classifier.

This phenomenon is further illustrated by the learning curves in
Figure~\ref{fig: random_label forget learning curve}, using the random-label unlearning method an an example. The output-level forget accuracy drops to nearly zero within the first few iterations, whereas the linear-probe and NCC accuracies remain almost unchanged throughout the entire unlearning process. This suggests that unlearning primarily suppresses the output classifier early in training, while the underlying feature representations of the forgotten class remain largely preserved and linearly separable.

Remarkably, the linear separability of forget class representations also persists in the model trained only on the retain set (denoted as Retain-only Retrain in Table~\ref{tab:ce_linearprobe}), a baseline commonly regarded as the “gold standard’’ for MU. Such a model achieves zero forget accuracy simply because it has not seen the forget data, yet linear probing still yields substantial forget accuracy. This is largely due to the transferability of DNN representations—typically viewed as one of their main advantages. In the context of unlearning, transferability makes it challenging to truly remove information at the
representation level.

Hence in this paper we primarily evaluate unlearning algorithms on their feature-level forget and retain accuracies. Since the “gold standard’’ model already contains the transferable features, this suggests a potential trade-off between output-level and feature-level retain and forget accuracies. This opens the door to methods that can outperform Retain-only Retrain by performing worse on output level metrics while being less transferable. For example, NegGrad+ achieves lower forget accuracy on CIFAR-10 when forgetting one class, albeit with a slight drop in retain accuracy. Nevertheless, as shown in Table~\ref{tab:ce_linearprobe}, such cases are rare, and Retain-only Retrain typically achieves lower feature-level forget accuracies overall. In \Cref{sec:cmf-unlearning}, we will introduce new MU methods that can remove information from hidden representations and consistently achieve lower feature-level forget accuracies.

\subsection{Feature-classifier Misalignment by \NC\ Analysis} \label{subsec:nc_analysis}

In the previous section, we have shown that current unlearning methods allow us to obtain zero forget accuracy while maintaining comparable accuracy on the retain set. However, performance on the forget set can be recovered with simple linear probing. What is the mechanism of unlearning that explains this observation? In this subsection we explain this illusion of unlearning through the lens of Neural Collapse (\NC).

In collapsed models, the last layer features of samples within the same class are concentrated around their class means (\NC$_1$), the class means form a simplex ETF (\NC$_2$), the classifier weights align with the class means (\NC$_3$), and the NCC rule at the last layer agrees with the decision of the deep network (\NC$_4$).

\begin{table*}[ht!]
\centering
\caption{ 
Full model VS classifier-only unlearning evaluated via mean output-level forget and retain accuracies.
}
\label{tab:unlearn_lastlayer}
\vspace{-.15in}
\resizebox{1.0\linewidth}{!}{
\begin{tabular}{l l cc cc cc cc}
\toprule
\multirow{3}{*}{\textbf{Method}}
  & \multirow{3}{*}{\shortstack{\textbf{Layers}\\\textbf{Finetuned}}}
  & \multicolumn{4}{c}{\textbf{CIFAR-10}}
  & \multicolumn{4}{c}{\textbf{CIFAR-100}} \\
\cmidrule(lr){3-6}\cmidrule(lr){7-10}
  &  & \multicolumn{2}{c}{1} & \multicolumn{2}{c}{3}
     & \multicolumn{2}{c}{1} & \multicolumn{2}{c}{10} \\
\cmidrule(lr){3-4}\cmidrule(lr){5-6}\cmidrule(lr){7-8}\cmidrule(lr){9-10}
  &  & Retain & Forget & Retain & Forget & Retain & Forget & Retain & Forget \\
\midrule
Original
  & Full Model             & 93.98 & 93.98 & 94.00 & 93.94 & 74.61 & 74.40 & 74.47 & 75.88 \\

\midrule
Retain-only Retrain
  & Full Model             & 94.74 & 0.00 & 95.37 & 0.00 & 76.01 & 0.00 & 76.50 & 0.00  \\

\midrule
\multirow{2}{*}{Retain-only FT}
  & Full Model             & 94.26 & 47.67 & 95.24 & 52.48 & 74.08 & 53.20 & 74.53 & 64.96\\
  & Classifier only        & 93.20 &  0.00 & 93.66 &  0.00 & 73.29 &  0.00 & 73.51 &  0.00 \\

\midrule
\multirow{2}{*}{NegGrad+}
  & Full Model             & 92.85 &  0.00 & 93.29 & 0.01 & 69.90 & 0.00 & 70.80 & 0.28  \\
  & Classifier only        & 93.08 &  0.00 & 93.72 &  0.00 & 73.28 &  0.00 & 66.85 &  0.00 \\

\midrule
\multirow{2}{*}{Random-label}
  & Full Model             & 92.93 &  0.00 & 94.14 &  0.00 & 72.33 & 0.00 & 72.19 & 0.00 \\
  & Classifier only        & 93.39 &  0.00 & 94.56 &  0.00 & 73.87 &  0.20 & 74.09 &  2.32 \\
\midrule
\multirow{2}{*}{Salun}
  & Full Model             & 93.19 &  0.00 & 94.43 &  0.00 & 72.96 & 0.00 & 72.92 & 0.06 \\
  & Classifier only        & 93.38 &  0.00 & 94.44 &  0.00 & 73.93 &  1.00 & 73.98 &  4.18 \\
\midrule
\multirow{2}{*}{SVD}
  & Full Model             & 92.02 &  0.00 & 94.05 & 57.43 &71.09 &  0.00 & 73.10 & 55.56 \\
  & Classifier only        & 93.53 &  0.01 & 93.73 & 68.45 & 73.47 &  2.60 & 74.13 & 50.56 \\
  \midrule
\multirow{2}{*}{SCRUB}
  & Full Model            & 91.37 & 0.00 & 93.61 &  0.00 & 73.67 & 0.20 & 74.56 & 0.32 \\
  & Classifier only        & 93.12 &  0.00 & 93.68 & 0.00 & 74.01 &  5.00 & 74.01 & 0.00 \\
  \midrule
\multirow{2}{*}{UNSIR}
  & Full Model             & 91.84 & 0.48 & 92.87 & 0.01 & 73.77 & 3.80 & 73.58 & 14.58  \\
  & Classifier only        & 94.48 & 0.00 & 94.46 & 1.39 & 74.64 & 0.00 & 74.27 & 0.52  \\
\bottomrule
\end{tabular}
}
\vspace{-.2in}
\end{table*}

 \begin{figure}[t]
        \centering
        \includegraphics[width=\linewidth]{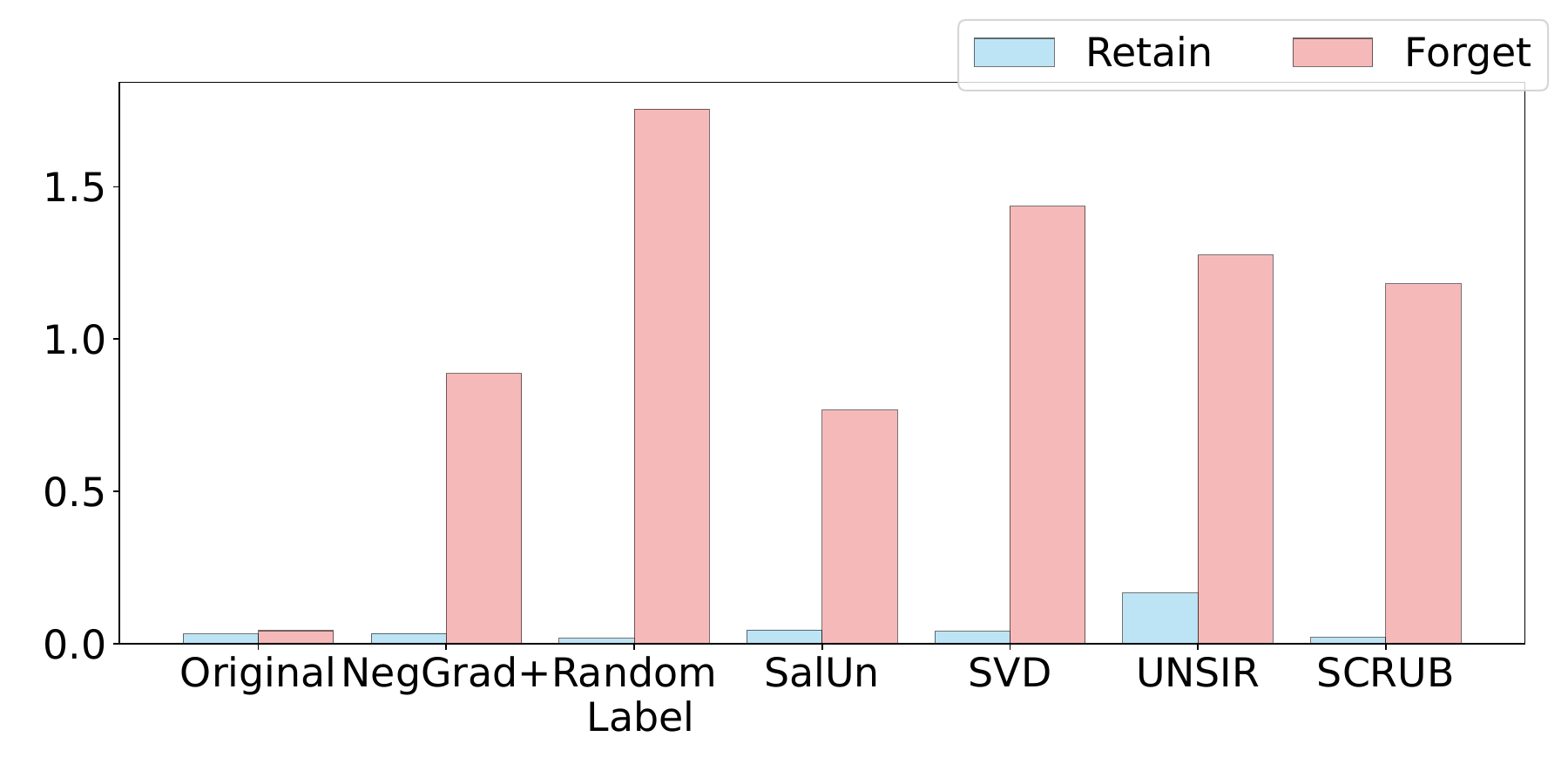}
    \vspace{-.4in}    
        \caption{Feature-classifier Alignment (\NC$_3$) for single class forgetting on CIFAR-10:  distance between class means and classifier weights for forget class is increased while the distance is preserved for retain class.}%
        \label{fig:cifarnc3}
        \vspace{-.2in}
 \end{figure}

\vspace{-.1in}
\paragraph{Observation 2: The illusion of unlearning is primarily caused by feature-classifier misalignment.}
The NCC accuracy (\NC$_4$) reported in Table~\ref{tab:ce_linearprobe} shows that the NCC classifier also achieves high forget accuracy across many unlearning methods. Since NCC classification depends only on distances between features and class means, this indicates that the feature representations of the forgotten classes remain clustered around their class means even after unlearning. In other words, the discriminative structure of the feature space is largely preserved.
We next explore how models that have their last layer features clustered around their class means still have near-zero forget accuracy. To measure this, we calculated the feature-weight alignment between the last layer features and the classifier weights (\NC$_3$). As shown in Figure~\ref{fig:cifarnc3}, the alignment for the retain classes is largely preserved after unlearning, whereas the classifier weight corresponding to the forget class becomes significantly misaligned with its class mean, which we term as {\it feature-classifier misalignment}. This shows that the model achieves zero forget accuracy by only shifting the last layer weights in an appropriate manner. 

In fact, under the assumptions of \NC\  and fixed class means, we can show that the optimal configuration of last layer weights after unlearning using NegGrad flips the forget classifier vector to be maximally misaligned with the forget class mean while minimally shifting the weights of the retain classifier vectors (a similar argument holds for Random label unlearning as well). We prove the following proposition in Appendix \ref{app:last-layer-analysis}.

\begin{prop} \label{prop:nc_unlearn_inf}

    Let $f_{\mb{\mb{W}, \vct{\theta}}} (\vx) = \mb{W} \phi_{\vct{\theta}}(\vx)$ be a classification model trained to collapse with last layer class mean features $\{\mb{\mu}_k\}_{k=1}^K$ that form a simplex equiangular tight frame, and assume that the mean features do not change during unlearning. If class $k \in [K]$ is unlearned using the NegGrad objective, then the resulting weights satisfy $\mb{w}^{\textrm{un}}_k = - (1- \gamma) \mb{\mu}_k$ for the forget class $k$, and $\mb{w}^{\textrm{un}}_i \propto \alpha \mb{\mu}_i + \beta \mb{\mu}_k$ for the retain classes $i \neq k$ where $0<\alpha, \beta, \gamma<1$.
\end{prop}

We depict this optimal configuration with the misaligned classifier vector in \Cref{fig:prop_1_vis}. This configuration of last layer weights also achieves zero output-level forget accuracy.

\begin{cor} \label{cor:nc_zero_forget}

Consider the same setting as proposition \ref{prop:nc_unlearn_inf}. Let $\hat{y}(\mb{x}) = \arg\max_{k \in [K]} \cos (\angle \mb{w}_k^\textrm{un}, \phi_{\vct{\theta}}(\mb{x}))$ be the prediction of the model where class $k$ has been unlearned. This model achieves zero forget accuracy, i.e., $\hat{y}(\mb{x}_{i,k}) \neq k$ for training samples $\mb{x}_{i,k}$ that belong to class $k$. 
\end{cor}
\begin{proof}
    Recall that in the \NC\ setting the last layer features are mapped to the fixed class means $\{\mb{\mu}_c\}_{c=1}^K$, and the class means form a simplex ETF, i.e. $\mb{M}^\top \mb{M} = \left( \mb{I}_K - \frac{1}{K} \mb{1}_K\mb{1}_K^\top \right)$ where $\mb{M}$ is a matrix with the columns set to the class means. This means that training samples in class $c$ are mapped to $\mb{\mu}_c, \forall c \in [K]$. For samples in the forget class $k$ we have
    \begin{equation*}
        \begin{split}
        \mb{w}_i^{\textrm{un} \top} \mb{\mu}_k &= (\alpha \mb{\mu}_i + \beta \mb{\mu}_k)^\top \mb{\mu}_k = \frac{-\alpha}{K-1} + \beta \\
        \implies \cos (\angle \mb{w}_c^\textrm{un}, \mb{\mu}_k) &= \frac{\frac{-\alpha}{K-1} + \beta} {\sqrt{\alpha^2 + \beta^2 - \frac{2\alpha \beta}{K-1}}} > -1, \forall c\neq k \\
        \mb{w}_k^{\textrm{un} \top} \mb{\mu}_k &= -(1-\gamma)\mb{\mu}_k^\top \mb{\mu}_k = \gamma - 1 \\
        \implies \cos (\angle \mb{w}_k^\textrm{un}, \mb{\mu}_k) &= -1
        \end{split}
    \end{equation*}
    This means that $\arg\max_{c \in [K]} \cos (\angle \mb{w}_c^\textrm{un}, \mb{\mu}_k) \neq k$, and hence the model achieves zero forget accuracy.

 \vspace{-.2in}   
\end{proof}

In order to confirm our hypothesis that last layer unlearning is sufficient to get zero forget set accuracy, we run experiments where we update only the last layer weights during unlearning and show the results in Table~\ref{tab:unlearn_lastlayer}. This leads us to our next observation.

\begingroup
\setlength{\tabcolsep}{4pt}
\small
\begin{table*}[t]
\centering
\caption{Evaluation of MU methods with CMF classifiers for unlearning certain number of classes. %
}
\label{tab:eval_resnet18}
\vspace{-0.12in}
\resizebox{1.0\linewidth}{!}{
\begin{tabular}{l l cccc cccc cccc}
\toprule
\multirow{3}{*}{\textbf{Method}} & \multirow{3}{*}{\textbf{Accuracy}} & \multicolumn{4}{c}{\textbf{CIFAR-10}} & \multicolumn{4}{c}{\textbf{CIFAR-100}} & \multicolumn{4}{c}{\textbf{Tiny-ImageNet}} \\
\cmidrule(lr){3-6}\cmidrule(lr){7-10}\cmidrule(lr){11-14}
 &  & \multicolumn{2}{c}{\textbf{1}} & \multicolumn{2}{c}{\textbf{3}} & \multicolumn{2}{c}{\textbf{1}} & \multicolumn{2}{c}{\textbf{10}} & \multicolumn{2}{c}{\textbf{1}} & \multicolumn{2}{c}{\textbf{20}}\\
\cmidrule(lr){3-4}\cmidrule(lr){5-6}\cmidrule(lr){7-8}\cmidrule(lr){9-10}\cmidrule(lr){11-12}\cmidrule(lr){13-14}
 &  & \textbf{Retain} & \textbf{Forget} & \textbf{Retain} & \textbf{Forget} & \textbf{Retain} & \textbf{Forget} & \textbf{Retain} & \textbf{Forget}& \textbf{Retain} & \textbf{Forget} & \textbf{Retain} & \textbf{Forget} \\
 \midrule
 \multirow{3}{*}{Original} & Output & 93.98 & 93.98 & 94.00 & 93.94 & 74.61 & 74.40 & 74.47 & 75.88 & 65.27 & 58.80 & 65.15 & 66.02 \\
 & Linear Probe & 94.02 & 94.02 & 94.03 & 94.00 & 74.53 & 75.00 & 74.38 & 75.90 & 65.10 & 60.80 & 64.97 & 66.08 \\
 & NCC & 94.00 & 93.99 & 94.03 & 93.92 & 74.40 & 75.00 & 74.28 & 75.68 & 64.65 & 60.00 & 64.56 & 65.00 \\
 \midrule
 \multirow{3}{*}{Retain-only Retrain} & Output & 94.74 & 0.00 & 95.37 & 0.00 & 76.01 & 0.00 & 76.50 & 0.00 & 66.52 & 0.00 & 66.38 & 0.00 \\
 & Linear Probe& 90.49 & 77.35 & 85.64 & 67.33 & 74.09 & 85.20 & 69.34 & 60.94 & 65.90 & 46.40 & 65.21 & 30.36 \\
 & NCC & 93.31 & 47.06 & 91.37 & 37.07 & 73.90 & 70.40 & 70.98 & 43.18 & 63.57 & 71.20 & 59.37 & 44.30 \\
\midrule
\multirow{3}{*}{Random-label with CMF} & Output  & 94.27 & 80.70 & 94.81 & 75.69 & 74.38 & 55.60 & 74.85 & 54.40 & 62.01 & 22.40 & 62.25 & 22.20\\
 & Linear Probe & 94.19 & 85.71 & 94.49 & 81.47 & 74.63 & 59.20 & 74.98 & 66.44 & 62.56 & 32.80 & 62.38 & 36.58 \\
 & NCC & 94.25 & 82.04 & 94.73 & 76.74 & 74.49 & 59.20 & 74.82 & 60.24 & 61.96 & 27.20 & 61.93 & 28.24 \\
\midrule
\multirow{3}{*}{Salun with CMF} & Output & 94.33 & 78.07 & 95.01 & 75.96  & 74.62 & 60.40 & 74.98 & 62.10  
& 62.62 & 30.00 & 63.14 & 32.74 \\
 & Linear Probe & 94.26 & 84.68 & 94.60 & 83.39 & 74.79 & 59.40 & 75.18 & 67.20 & 63.18 & 41.20 & 63.12 & 46.14 \\
 & NCC & 94.32 & 79.50 & 94.91 & 77.31 & 74.78 & 63.60 & 75.03 & 65.20 & 62.67 & 37.20 & 62.77 & 38.10 \\
\midrule
\multirow{3}{*}{NegGrad+ with CMF} & Output & 91.87 & 54.50 & 94.97 & 60.00 & 71.82 & 41.00 & 71.45 & 30.38 & 61.18 & 22.00 & 60.82 & 41.42\\
 & Linear Probe & 92.35 & 68.02 & 94.60 & 75.29 & 72.86 & 55.20 & 72.35 & 50.62 & 62.90 & 41.20 & 62.66 & 53.82 \\
 & NCC & 91.99 & 57.83 & 94.93 & 62.89 & 72.36 & 51.20 & 71.75 & 41.54 & 62.31 & 38.80 & 62.11 & 51.30 \\
 \midrule
\multirow{3}{*}{Scrub with CMF} & Output  & 92.51 & 33.78 & 95.37 & 35.11 & 73.86 & 40.60 & 74.27 & 35.02 & 61.64 & 27.20 & 63.31 & 47.76 \\
 & Linear Probe  & 92.48 & 60.68 & 95.26 & 62.77 & 74.03 & 55.60 & 74.18 & 58.34 & 62.13 & 36.80 & 63.76 & 54.28 \\
 & NCC & 92.48 & 35.53 & 95.34 & 40.04 & 73.87 & 47.00 & 74.12 & 42.82 & 61.66 & 34.40 & 63.28 & 50.42 \\

 \midrule
\multirow{3}{*}{UNSIR with CMF}& Output & 91.79 & 12.91 & 93.56 & 11.51 & 72.72 & 21.00 & 72.61 & 9.16 & 60.81 & 14.00 & 61.34 & 14.44 \\
 & Linear Probe & 91.63 & 31.16 & 93.01 & 28.65 & 72.91 & 35.20 & 72.51 & 20.98 & 60.96 & 26.80 & 61.27 & 23.02 \\
 & NCC & 91.87 & 9.29 & 93.79 & 8.16 & 72.67 & 20.20 & 72.48 & 9.94 & 60.63 & 26.00 & 60.80 & 16.72 \\
\bottomrule
\end{tabular}
}
\end{table*}
\vspace{-9pt}
\endgroup

\begin{figure*}[ht]
\centering

\begin{subfigure}[t]{0.19\linewidth}
  \centering
  \includegraphics[width=\linewidth]{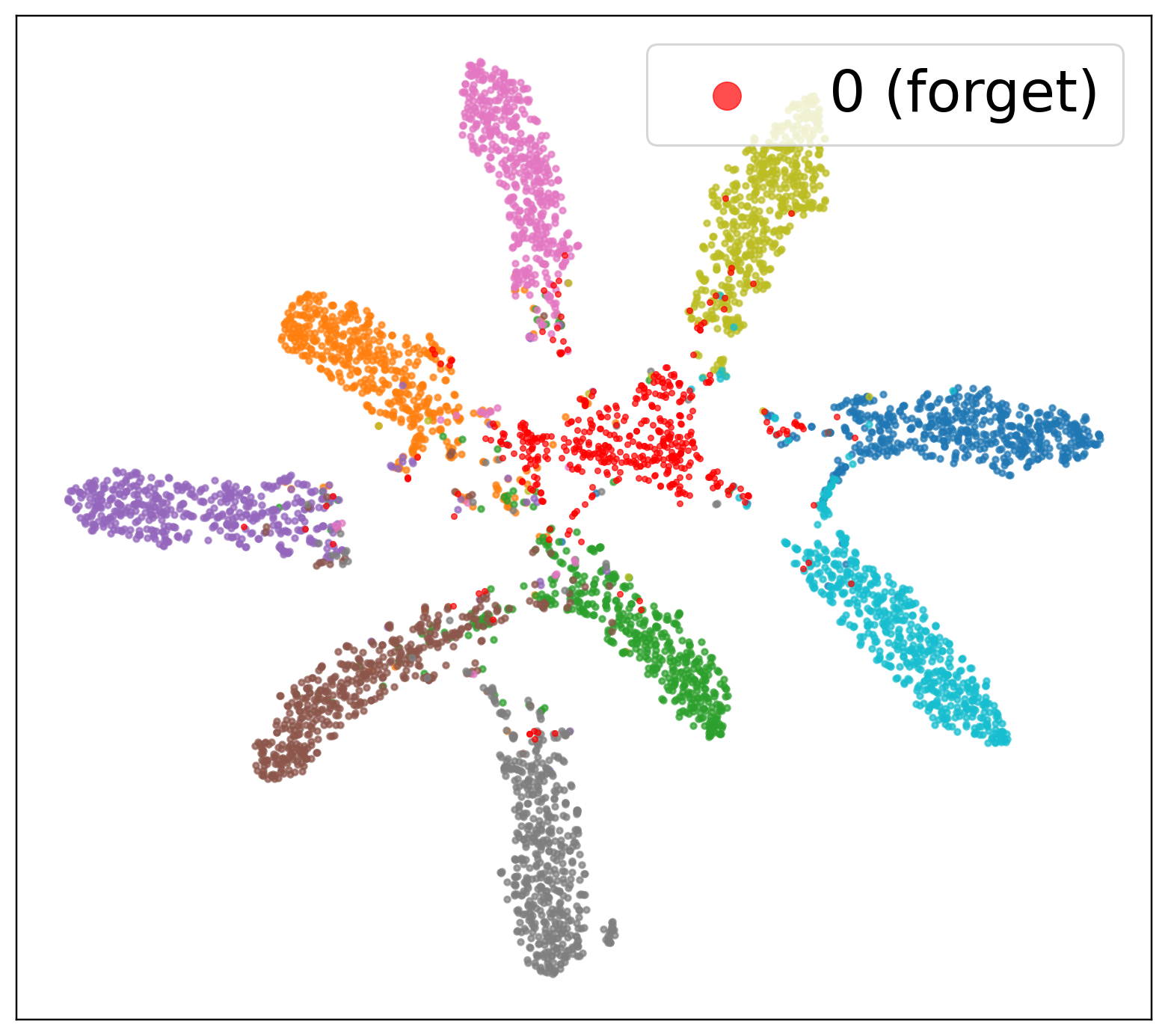}
  \caption{Random-label}
\end{subfigure}\hfill
\begin{subfigure}[t]{0.19\linewidth}
  \centering
  \includegraphics[width=\linewidth]{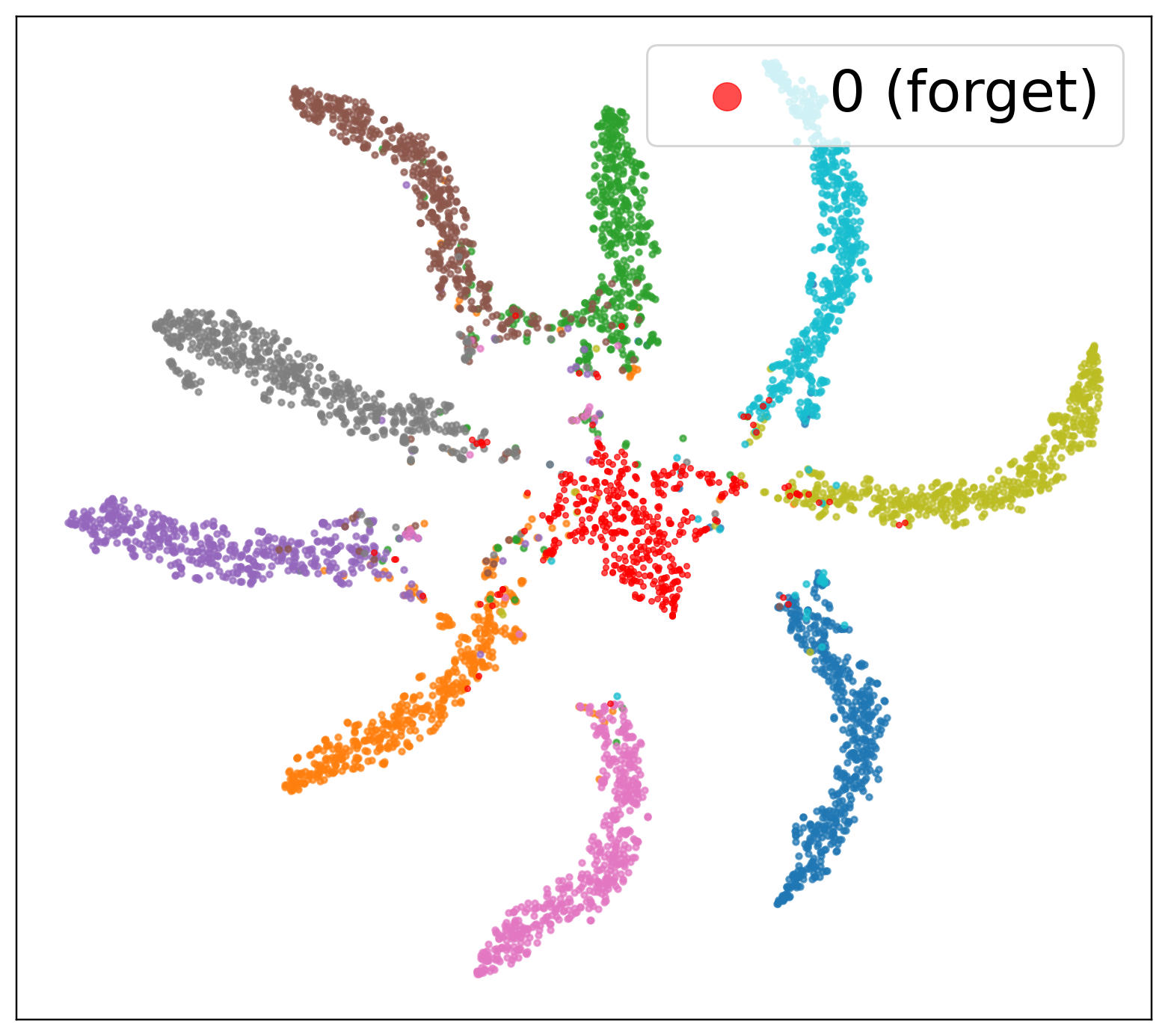}
  \caption{SalUn }
\end{subfigure}\hfill
\begin{subfigure}[t]{0.19\linewidth}
  \centering
  \includegraphics[width=\linewidth]{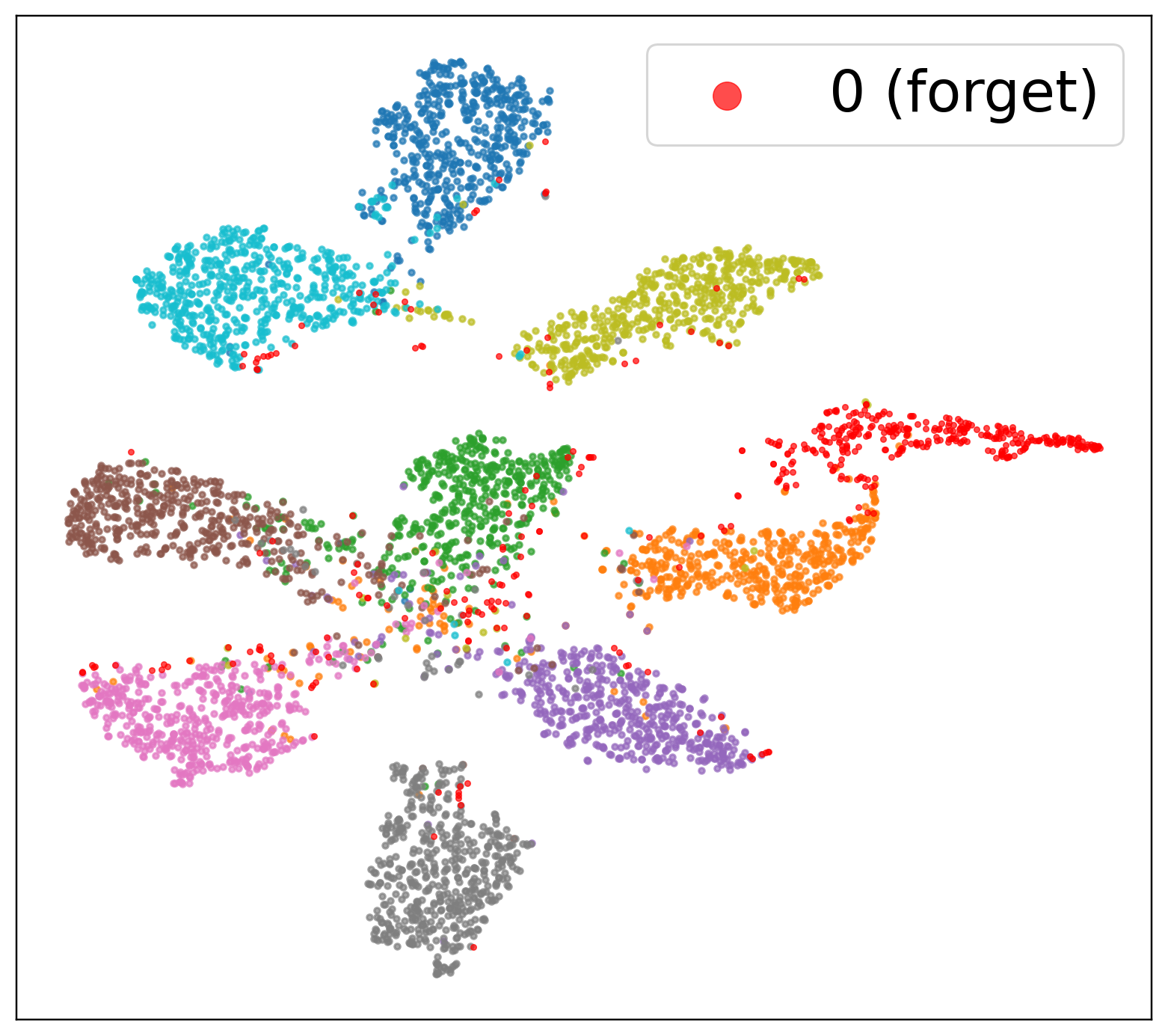}
  \caption{NegGrad+}
\end{subfigure}\hfill
\begin{subfigure}[t]{0.19\linewidth}
  \centering
  \includegraphics[width=\linewidth]{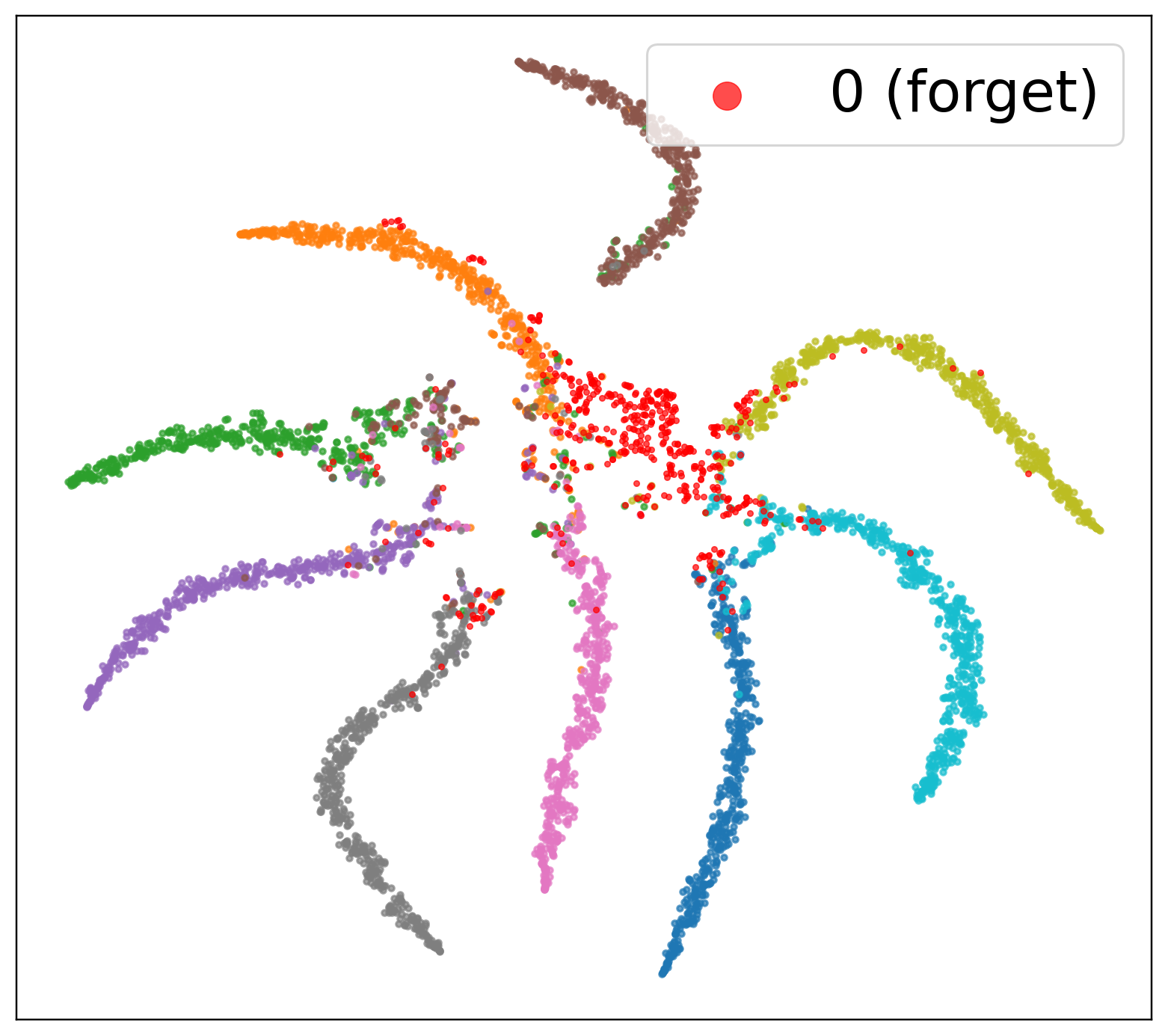}
  \caption{SCRUB}
\end{subfigure}\hfill
\begin{subfigure}[t]{0.19\linewidth}
  \centering
  \includegraphics[width=\linewidth]{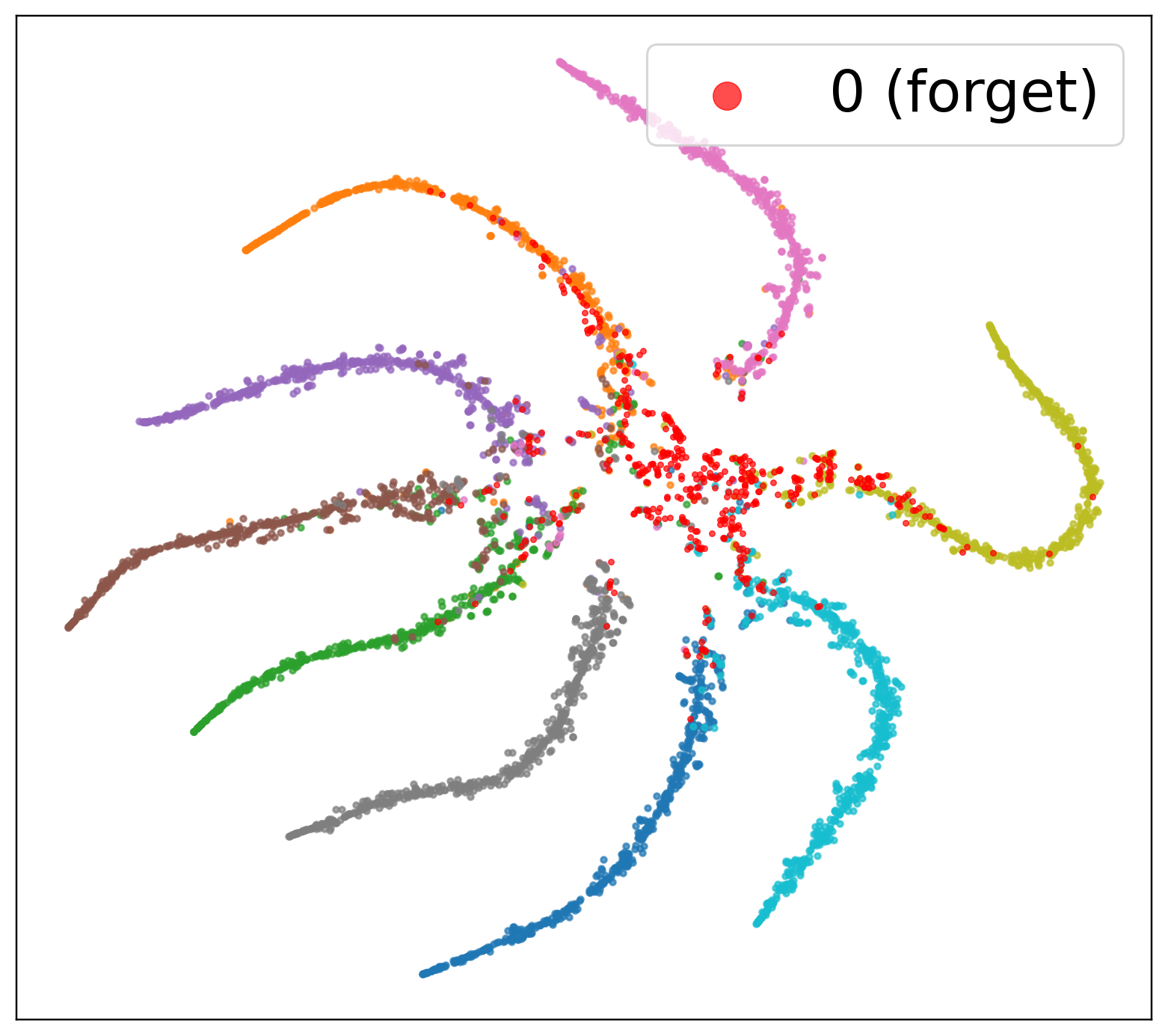}
  \caption{UNSIR}
\end{subfigure}
\vspace{-9pt}
\caption{t-SNE of features learned with CMF-based unlearning methods. The forgotten class (red points) exhibits a distribution that is markedly different from the retain classes and shows noticeable overlap with them. Here the overall feature distributions are reshaped due to an additional normalization step. 
}
\vspace{-13pt}
\label{fig:tsne_cmf_all_methods}
\end{figure*}

\paragraph{Observation 3: Classifier-only unlearning achieves comparable 
performance at the output level.}
As shown in Table~\ref{tab:unlearn_lastlayer}, updating only the classifier during unlearning achieves performance that is comparable to, or even slightly better than, full-model unlearning in terms of both forgetting and retaining accuracy across different MU methods for most scenarios. However, although models appear to forget at the output level, the feature mappings remain unchanged from the original model and continue to encode information about the forgotten classes. This finding underscores that current MU methods primarily achieve output suppression rather than representation erasure, challenging the validity of evaluating unlearning effectiveness solely through output-level metrics.

\subsection{Class-Mean-Features Unlearning}
\label{sec:cmf-unlearning}
To enable unlearning in the feature space, we propose representation-level unlearning methods that employ class-mean features (CMF) classifiers to address the classifier–feature misalignment described above. Specifically, inspired by the self-duality between features and classifiers, the CMF classifier was originally proposed in \citep{jiang2024generalized} to reduce trainable parameters by setting classifier weights to the exponential moving average of the mini-batch class-mean features during training. In our work, we adapt CMF classifiers to the unlearning setting, leveraging them to enforce alignment between classifiers and features throughout the unlearning process.

Formally, we construct the CMF classifier via 
\begin{equation}
\mW_{\mathrm{CMF}} \;=\;
\begin{bmatrix}
{\vmu}_1 & \cdots & 
{\vmu}_K
\end{bmatrix}^\top
\in\mathbb{R}^{K\times d},
\label{eq:cmf-head}
\end{equation}
where $\vmu_c$ denotes the mean feature vector for class $c$ as defined in \eqref{eq:means} and  can be updated at each epoch during training. The CMF classifier can be seamlessly integrated into existing MU methods by replacing $\mW$ with $\mW_{\mathrm{CMF}}$ in \eqref{eq:DNN} and plugging the resulting model into the general MU objective in \eqref{eq:general_mu}.

We apply the CMF classifier to multiple representative MU methods, including Random Label, SalUn, NegGrad+, SCRUB, and UNSIR, with mean results summarized in Table~\ref{tab:eval_resnet18}. 
By explicitly enforcing alignment between classifiers and features, the unlearning process becomes more challenging: CMF-based methods no longer trivially achieve zero forget accuracy at the output level. Nevertheless, model features now encode substantially less information about the forgotten classes, as indicated by the much lower feature-level forget accuracy under linear probing. Remarkably, the proposed CMF-enhanced unlearning methods consistently achieve significantly lower feature-level forget accuracy than the Retain-only Retrain baseline, while incurring only a very mild decrease in retain accuracy. This result highlights the strength of CMF in mitigating feature–classifier misalignment, ensuring that forgetting occurs not just in predictions but also within the hidden representations. We can
qualitatively observe this in the t-SNE plots of  Figure~\ref{fig:tsne_cmf_all_methods}.
Overall, our findings underscore that CMF provides a principled and effective framework for representation-level unlearning, offering a more faithful approach to removing information from deep models compared to existing baselines.

\section{Conclusion and Future Work}

In this paper we describe an {\it illusion of unlearning} where models appear to forget classes when
evaluated at the output level, while still
retaining information about the forgotten data in their hidden
representations. 
We demonstrate that training linear probes on features from unlearned models can recover performance on the forget set. Through an \NC\ analysis, we observe that the unlearning methods mainly alter the final classifier weights to be misaligned to the forget classes, while maintaining the representations of the forget classes in layers below the last layer. To mitigate this issue, we propose {\it class-mean-features unlearning}, which ties classifier weights to
class-mean features and encourages the removal of forgotten information from the representation space.

There are several promising directions for future work. First is the extension to generative diffusion and language models where our shallow unlearning observations are aligned with recent findings 
such as the fact that simple fine-tuning can inadvertently reintroduce erased concepts \citep{suriyakumar2024unstable}. 
Next, the transferability of features in deep learning poses challenges for unlearning. We would like to characterize the trade-off between removing feature-level information about the forget data while maintaining performance on the retained data.
Finally, neural collapse phenomena have
also been observed in intermediate layers~\citep{rangamani2023feature},
and future work may investigate whether extending CMF-style constraints
to deeper layers can further improve representation-level unlearning.

\section*{Acknowledgement}
YG and ZZ acknowledge support from NSF grants IIS-2312840 and IIS-2402952. We gratefully acknowledge Jinxin Zhou and Huminhao Zhu for valuable discussions.

\bibliography{references}
\bibliographystyle{apalike}

\newpage
\clearpage
\section*{Checklist}

\begin{enumerate}

  \item For all models and algorithms presented, check if you include:
  \begin{enumerate}
    \item A clear description of the mathematical setting, assumptions, algorithm, and/or model. [Yes] 
    
    See Sections~\ref{sec:background_DNN and NC} and~\ref{sec:cmf-unlearning} for the main description.  
Algorithm~\ref{alg:cmf_head}, Algorithm~\ref{alg:cmf_generic} in Appendix~\ref{app:algorithms} provide detailed pseudocode for the CMF unlearning strategies.
    \item An analysis of the properties and complexity (time, space, sample size) of any algorithm. [Yes]  

See Appendix~\ref{app:complexity} for a discussion of the time, space, and sample complexity of CMF-based unlearning.
    \item (Optional) Anonymized source code, with specification of all dependencies, including external libraries. [Yes] 
    
    An anonymized version of the source code with all dependencies (e.g., PyTorch, PyTorch Lightning, Torchmetrics, NumPy) will be released upon acceptance.
  \end{enumerate}

  \item For any theoretical claim, check if you include:
  \begin{enumerate}
    \item Statements of the full set of assumptions of all theoretical results. [Yes]

    See Section~\ref{sec:background_DNN and NC} for assumptions on Neural Collapse and class mean features, and Section~\ref{sec:cmf-unlearning} for assumptions underlying the unlearning analysis.
    
    \item Complete proofs of all theoretical results. [Yes]
    \item Clear explanations of any assumptions. [Yes]     
  \end{enumerate}

  \item For all figures and tables that present empirical results, check if you include:
  \begin{enumerate}
    \item The code, data, and instructions needed to reproduce the main experimental results (either in the supplemental material or as a URL). [Yes]
    \item All the training details (e.g., data splits, hyperparameters, how they were chosen). [Yes]
    \item A clear definition of the specific measure or statistics and error bars (e.g., with respect to the random seed after running experiments multiple times). [Yes]
    \item A description of the computing infrastructure used. (e.g., type of GPUs, internal cluster, or cloud provider). [Yes]
  \end{enumerate}

  \item If you are using existing assets (e.g., code, data, models) or curating/releasing new assets, check if you include:
  \begin{enumerate}
    \item Citations of the creator If your work uses existing assets. [Yes]
    \item The license information of the assets, if applicable. [Yes]
    \item New assets either in the supplemental material or as a URL, if applicable. [Yes]
    \item Information about consent from data providers/curators. [Yes] The datasets that
  we use are public.
    \item Discussion of sensible content if applicable, e.g., personally identifiable information or offensive content. [Not Applicable]
  \end{enumerate}

  \item If you used crowdsourcing or conducted research with human subjects, check if you include:
  \begin{enumerate}
    \item The full text of instructions given to participants and screenshots. [Not Applicable]
    \item Descriptions of potential participant risks, with links to Institutional Review Board (IRB) approvals if applicable. [Not Applicable]
    \item The estimated hourly wage paid to participants and the total amount spent on participant compensation. [Not Applicable]
  \end{enumerate}

\end{enumerate}

\clearpage
\thispagestyle{empty}

\onecolumn
\aistatstitle{An Illusion of Unlearning? Assessing Machine Unlearning Through Internal Representations: Appendix}
\label{app:appendix}
\appendix
\section{Experimental Settings}
\label{app:exp_settings}

\subsection{Experiment Environment}

All experiments are conducted on a Linux server running Ubuntu~20.04 (kernel 5.4), equipped with 8 NVIDIA RTX A5000 GPUs (24\,GB memory each), 
64 CPU cores, and 252\,GB RAM. 
The software environment uses Python~3.11, PyTorch~2.5.1, and CUDA~12.1.

\subsection{Datasets}

We evaluate class-unlearning on three standard image classification benchmarks: CIFAR-10, CIFAR-100, and Tiny-ImageNet.
CIFAR-10 contains 10 classes with 50,000 training images and 10,000 test images.
CIFAR-100 has the same number of images but with 100 classes.
Tiny-ImageNet contains 200 classes with 500 training images and 50 validation images per class.

For original model training, we use train/validation/test splits for model selection and evaluation.
For unlearning experiments, we use the standard train/test split and report performance on the test set.

\subsection{Unlearning Scenarios}\label{unlearning_scenrios}

We evaluate unlearning algorithms under both\emph{single-class} and \emph{multi-class}
forgetting scenarios on each dataset.
For every dataset and scenario, we construct multiple retain–forget dataset
combinations by selecting different class indices as the forget set.
Each unlearning algorithm is evaluated on 5–10 such combinations, depending
on the dataset and setting.
For each experiment group, we report the mean and standard deviation of the
evaluation metrics across all combinations.

\textbf{CIFAR-10.}
For CIFAR-10, we evaluate both single-class and three-class forgetting.
In the single-class setting, we sweep across all classes
($\{0\}, \{1\}, \ldots, \{9\}$).
In the three-class setting, we evaluate several representative class
combinations such as
$\{0,1,2\}$, $\{3,4,5\}$, $\{6,7,8\}$, $\{0,5,9\}$, and $\{2,4,8\}$.

\textbf{CIFAR-100.}
CIFAR-100 consists of 100 fine-grained classes grouped into 20 coarse classes,
each containing 5 fine classes.
For single-class unlearning, we evaluate several representative classes
$\{0\}, \{1\}, \{2\}, \{3\}, \{5\}$.
These classes are selected to cover different coarse classes.
Notably, classes 1 and 4 belong to the same coarse class, and therefore
we evaluate only one of them to avoid redundant experiments within the same superclass.

For multi-class unlearning, we remove 10 classes at a time.
Each such setting corresponds to the union of two coarse classes
(i.e., $2\times5$ fine classes).
We evaluate multiple such combinations to cover different regions of the label space.
The specific class sets used in our experiments include:
$\{3,15,19,21,31,38,42,43,88,97\}$,
$\{47,52,54,56,59,62,70,82,92,96\}$,
$\{5,20,22,25,39,40,84,86,87,94\}$,
$\{8,13,41,48,59,69,81,85,89,90\}$,
and
$\{1,4,30,32,55,67,72,73,91,95\}$.

\textbf{Tiny-ImageNet.}
Tiny-ImageNet contains 200 classes.
We evaluate both single-class forgetting and larger group unlearning settings.

For single-class forgetting, we evaluate several representative classes,
including $\{2\}, \{3\}, \{5\}, \{7\}, \{9\}$.

For multi-class unlearning, we remove groups of 20 classes at a time.
The class groups are constructed as contiguous ranges of class indices,
including
$\{0,\ldots,19\}$,
$\{20,\ldots,39\}$,
$\{40,\ldots,59\}$,
$\{60,\ldots,79\}$,
and
$\{80,\ldots,99\}$.
These groups cover different regions of the label space.

\subsection{Original Model Training}\label{original_model_training}

We use ResNet and ViT models in our experiments.
For the ResNet experiments, we train all models from scratch. Specifically, we use ResNet-18 on CIFAR-10 and CIFAR-100, and ResNet-50 on Tiny-ImageNet.
For the ViT experiments, we use ViT-S/16 models initialized from ImageNet-pretrained weights and then fine-tune them on each target dataset, including CIFAR-10, CIFAR-100, and Tiny-ImageNet.

During training, we apply standard data augmentation, including random cropping and random horizontal flipping.

For ViT models, input images are resized to $224\times224$ resolution.

For ResNet training, we use batch size 128 and train for up to 300 epochs with early stopping (patience 50).  For optimization, we use SGD with momentum 0.9 and weight decay $5\times 10^{-4}$.
The initial learning rate is set to $5\times10^{-2}$.
We apply a learning-rate warmup for the first 5 epochs followed by cosine learning-rate decay with a minimum learning rate of $1\times10^{-5}$.

For ViT experiments, we use ViT-S/16 models initialized from ImageNet-pretrained weights and fine-tune them on each target dataset.
The models are fine-tuned for 10 epochs with batch size 128 and learning rate $3\times10^{-4}$.
We use the same optimizer configuration as the ResNet training.

The resulting full-data model serves as the \emph{original model}, which is used as the starting point for all subsequent unlearning algorithms.

\subsection{Retrain-on-Retain Baseline (Gold Standard)}

To establish a reference for unlearning methods, we retrain models from scratch using only the retain subset. This baseline is considered the “gold standard” for machine unlearning.

The retrain models follow the same training settings described in 
Appendix~\ref{original_model_training}.  For each unlearning scenario, we construct the retain dataset according to the corresponding retain–forget split defined in Appendix~\ref{unlearning_scenrios}, and train separate retrain models as references for comparison with unlearning methods. These models are reported as \textit{Retain-only Retrain} in the experimental results.

\subsection{Unlearning Methods}

We evaluate several representative unlearning algorithms:

\begin{itemize}
  \item \textbf{Random Label (RL) \citep{Golatkar_2020_CVPR}}  
  Forget-class samples are reassigned random labels from the retain set.

  \item \textbf{SalUn \citep{fan2023salun}}  
  SalUn perturbs important model parameters associated with the forget classes based on saliency scores.

  \item \textbf{NegGrad+ (Grad-Ascent-Descent) \citep{choi2023towards}}  
  NegGrad+ adjusting the model’s output on forget data by performing gradient ascent on forget samples and gradient descent on retain samples.

  \item \textbf{SCRUB \citep{kurmanji2023towards}}  
  SCRUB formulates unlearning as a teacher--student distillation problem. 
  The original model acts as a teacher and a student model is trained to match the teacher on retain data while diverging from the teacher on forget data.

  \item \textbf{UNSIR \citep{tarun2023fast}}  
  UNSIR performs unlearning through an impair--repair process. 
  First, an error-maximizing noise matrix is generated to maximize the loss for the target forget classes. 
  The model is then updated using this noise together with a subset of retain data, followed by additional training on the retain data only to recover the model's performance.

  \item \textbf{SVD (Training-Free) \citep{kodge2024deep}}  
  A training-free method that removesforget-class information by performing singular value decomposition (SVD) on class-specific feature activations to estimate retain and forget subspaces, and suppressing the forget-discriminative components in the model parameters.
\end{itemize}

The detailed training hyperparameters for all methods are summarized in Table~\ref{tab:unlearn-hparams_resnet} and Table~\ref{tab:unlearn_hparams_vit}.

\subsection{Summary of Hyperparameters}
Tables \ref{tab:unlearn-hparams_resnet} and \ref{tab:unlearn_hparams_vit} summarize the hyperparameters across unlearning methods on ResNet and ViT models.

\subsection{Additional Results with Standard Deviations}
Tables~\ref{tab:Resnet_Normal_with_var}--\ref{tab:eval_vit_s_16_with_var}
report the mean accuracy together with the standard
deviation over multiple runs for all evaluated methods.

\section{CMF Unlearning Algorithms}
\label{app:algorithms}

\subsection{CMF-based Unlearning Framework}

In this appendix, we provide pseudocode for the core components of our
CMF-based unlearning framework.
Our goal is to make the CMF head reconstruction compatible with a wide
range of existing machine unlearning methods.

The key idea is to decouple the classifier head construction from the
underlying unlearning objective.
Instead of updating the classifier weights through standard gradient
training, we reconstruct the classifier head directly from class mean
features at the beginning of each epoch.
The head is then kept frozen while the encoder is updated according to
the objective of the chosen unlearning method.

Algorithm~\ref{alg:cmf_head} describes the CMF head reconstruction
procedure.
Given an encoder and a dataset, we compute the feature mean for each
class, center the class means, and normalize them to obtain the
classifier weights.
The resulting CMF head captures the geometric structure of the feature
space and is fixed during the subsequent optimization steps.

Algorithm~\ref{alg:cmf_generic} illustrates how the reconstructed CMF
head can be integrated into a generic gradient based unlearning
pipeline.
At the beginning of each epoch, the CMF classifier is rebuilt from the
current encoder features.

\begin{algorithm}[h]
\caption{CMF Head Reconstruction}
\label{alg:cmf_head}
\begin{algorithmic}[1]
\REQUIRE Encoder $z_\theta(\cdot)$, dataset $D=\{(\mathbf{x}_i,y_i)\}_{i=1}^{n}$ with class set $K$
\FOR{each class $k\in K$}
    \STATE Compute class-mean feature
    $\boldsymbol{\mu}_k=\frac{1}{|D_k|}\sum_{(\mathbf{x}_i,y_i)\in D_k} z_\theta(\mathbf{x}_i)$
\ENDFOR
\STATE Compute global mean $\bar{\boldsymbol{\mu}}=\frac{1}{K}\sum_{k\in K}\boldsymbol{\mu}_k$
\FOR{each class $k\in K$}
    \STATE Center and normalize
    $\mathbf{w}_k = \frac{\boldsymbol{\mu}_k-\bar{\boldsymbol{\mu}}}{\|\boldsymbol{\mu}_k-\bar{\boldsymbol{\mu}}\|}$
\ENDFOR
\STATE Form classifier head $\mathbf{W}_{\mathrm{CMF}}=[\mathbf{w}_1,\dots,\mathbf{w}_K]$
\STATE Freeze $\mathbf{W}_{\mathrm{CMF}}$
\STATE \textbf{return} $\mathbf{W}_{\mathrm{CMF}}$
\end{algorithmic}
\end{algorithm}

\begin{algorithm}[h]
\caption{Gradient-Based Unlearning with CMF}
\label{alg:cmf_generic}
\begin{algorithmic}[1]
\REQUIRE Encoder $z_\theta(\cdot)$, dataset $D=\{(\mathbf{x}_i,y_i)\}_{i=1}^{n}$, epochs $E$, unlearning objective $\mathcal{L}_U$
\STATE Reconstruct CMF head
$\mathbf{W}_{\mathrm{CMF}} \gets \textsc{CMF Head Reconstruction}(z_\theta,D)$
\FOR{$e=1$ \TO $E$}
    \FOR{mini-batch $\mathcal{B}=\{(\mathbf{x}_i,y_i)\}$}
        \STATE Compute unlearning loss
        $\mathcal{L} = \mathcal{L}_U(\mathbf{W}_{\mathrm{CMF}}z_\theta(\mathbf{x}_i), y_i)$
        \STATE Update encoder parameters $\theta$
    \ENDFOR
    \STATE Reconstruct CMF head
    $\mathbf{W}_{\mathrm{CMF}} \gets \textsc{CMF Head Reconstruction}(z_\theta,D)$
\ENDFOR
\STATE \textbf{return} encoder $z_\theta$ and $\mathbf{W}_{\mathrm{CMF}}$
\end{algorithmic}
\end{algorithm}

The loss $\mathcal{L}_U$ in Algorithm~\ref{alg:cmf_generic} corresponds
to the objective used by the underlying unlearning method.
For example, $\mathcal{L}_U$ may correspond to the random-label
cross-entropy loss in Random Label, or the ascent–descent objective used in gradient-based
unlearning methods such as NegGrad+.
Therefore, CMF head reconstruction can serve as a modular component
that augments a wide range of existing unlearning approaches.

\subsection{Complexity Analysis}
\label{app:complexity}

The CMF head reconstruction described in Algorithm~\ref{alg:cmf_head} introduces only a small computational overhead compared to standard stochastic gradient descent (SGD) training.

At the beginning of each epoch, the class-mean features are computed by a forward pass of the encoder over all samples in the dataset $D$. Let $T_f$ denote the cost of one forward pass of the encoder. Computing the class means therefore requires $O(|D|T_f)$ time.

After obtaining the class means, reconstructing the CMF classifier head requires centering and normalizing the mean vectors, which costs $O(Cd)$ where $C$ is the number of classes and $d$ is the feature
dimension.

The additional memory overhead is $O(Cd)$ for storing the class-mean vectors, which is negligible compared to the parameters of the encoder.

Therefore, the overall computational cost per epoch consists of the
standard training cost plus one small additive cost from CMF head reconstruction.
Since the primary cost during training derives from the forward/backward propagation of the encoder, the overhead introduced by CMF reconstruction is small
in practice.

\section{Analysis of Last layer Unlearning}
\label{app:last-layer-analysis}

In section \ref{subsec:nc_analysis} we measure the Neural Collapse (NC) metrics for networks that have been unlearned and observe that while the classes in the retain and forget sets remain separable, the distance between the classifier and the class mean features increases. This, combined with our Linear Probing results suggests that class unlearning in deep networks is primarily achieved by changing the alignment of the classifier without changing the features. To analyze how the classifier changes during unlearning, we derive the minima of the Neg-Grad loss under the assumption that the original model was trained to collapse, and that the class mean features do not move.

\begin{prop} \label{prop:nc_unlearn}
    Let $f_{\mb{\mb{W}, \vct{\theta}}} (\vx) = \mb{W} \phi_{\vct{\theta}}(\vx)$ be a classification model trained to collapse with last layer class mean features $\{\mb{\mu}_k\}_{k=1}^K$ that form a simplex equiangular tight frame, and assume that the mean features do not change during unlearning. If class $k \in [K]$ is unlearned using the Neg-grad objective, then the resulting weights satisfy $\mb{w}^{\textrm{un}}_k = - (1- \gamma) \mb{\mu}_k$ for the forget class $k$, and $\mb{w}^{\textrm{un}}_i \propto \alpha \mb{\mu}_i + \beta \mb{\mu}_k$ for the retain classes $i \neq k$ where $0<\alpha, \beta, \gamma<1$.
\end{prop}

\begin{proof}
    The regularized Neg-grad objective is given by:
    \begin{equation}\label{eqn:neg-grad}
        \begin{split}
            \mc{L} &= \frac{1}{K-1} \sum_{i \neq k} \textrm{log} \left( \sum_{j=1}^K \textrm{exp} (\mb{w}_j^\top \mb{\mu}_i) \right) - \mb{w}_i^\top \mb{\mu}_i \\ &+ \mb{w}_k^\top \mb{\mu}_k - \textrm{log} \left( \sum_{j=1}^K \textrm{exp} (\mb{w}_j^\top \mb{\mu}_k) \right) + \frac{\lambda_W}{2} \| \mb{W} \|_F^2
        \end{split}
    \end{equation}

Consider the gradients of the objective $\mc{L}$ wrt the weights of the forget and retain classes:

\begin{equation*}
    \frac{\partial \mathcal{L}}{\partial \mb{w}_j} = \frac{1}{K-1} \sum_{i\neq k} \frac{\exp(\mb{w}_j^T \mb{\mu}_i)}{\Lambda_i} \mb{\mu}_i - \frac{\mb{\mu}_j}{K-1} - \frac{\exp(\mb{w}_j^T \mb{\mu}_k)}{\Lambda_k} \mb{\mu}_k + \lambda_W \mb{w}_j
\end{equation*}

\begin{equation*}
    \frac{\partial \mathcal{L}}{\partial \mb{w}_k} = \frac{1}{K-1} \sum_{i\neq k} \frac{\exp(\mb{w}_k^T \mb{\mu}_i)}{\Lambda_i} \mb{\mu}_i + \mb{\mu}_k - \frac{\exp(\mb{w}_k^T \mb{\mu}_k)}{\Lambda_k} \mb{\mu}_k + \lambda_W \mb{w}_k
\end{equation*}

Here $\Lambda_i$ denotes the logsumexp of the classifier scores for mean feature $\mb{\mu}_i$. From Lemma \ref{lm:stat_pts} we can observe that at stationary points of the objective $\mc{L}$ for all $j \neq k$ and $l \neq j, k$ we have that $\mb{w}_j ^\top \mb{\mu}_l := b$ are all equal, $\mb{w}_j ^\top \mb{\mu}_j := a$ are equal, and $\mb{w}_k ^\top \mb{\mu}_j := c$ are equal. Moreover, we also have that $\mb{w}_j^\top \mb{\mu}_k := b'$ are equal for $j \neq k$ and $\mb{w}_k^\top \mb{\mu}_k = c$. This means that $\Lambda_k = \exp(c) + (K-1)\exp(b')$ and $\Lambda_i = \exp(a) + (K-2) \exp(b) + \exp(c) = \Lambda$ are equal for all $i \neq k$. Plugging this into the above gradient expressions and computing the stationary points, we obtain for $j \neq k$:

\begin{equation}
    \begin{split}
        \lambda_W \mb{w}_j^{\textrm{un}} &= \frac{\mb{\mu}_j}{K-1} - \frac{1}{K-1} \left[ \frac{\exp(b)}{\Lambda} \sum_{i\neq j,k} \mb{\mu}_i + \frac{\exp(a)}{\Lambda} \mb{\mu}_j \right] + \frac{\exp(b')}{\Lambda_k} \mb{\mu}_k \\
        &= \frac{1}{K-1} \left( 1- \frac{\exp(b) - \exp(a)}{\Lambda} \right) \mb{\mu}_j + \left( \frac{\exp(b')}{\Lambda_k} + \frac{\exp(b)}{(K-1)\Lambda} \right) \mb{\mu}_k
    \end{split}
\end{equation}
Where we have used the simplex ETF condition to obtain $\sum_{i\neq j,k} \mb{\mu}_i = -\mb{\mu}_j - \mb{\mu}_k$.

For the forget class $k$ we have:

\begin{equation}
    \begin{split}
        \lambda_W \mb{w}_k^{\textrm{un}} &= -\mb{\mu}_k + \frac{\exp(\mb{w}_k^T \mb{\mu}_k)}{\Lambda_k} \mb{\mu}_k - \frac{1}{K-1} \sum_{i\neq k} \frac{\exp(\mb{w}_k^T \mb{\mu}_i)}{\Lambda_i} \mb{\mu}_i\\
        &= -\mb{\mu}_k + \frac{\exp(c)}{\Lambda_k} \mb{\mu}_k - \frac{1}{K-1} \frac{\exp(c)}{\Lambda} \sum_{i\neq k} \mb{\mu}_i\\
        &= - \left( 1 - \frac{\exp(c)}{\Lambda_k} - \frac{1}{K-1} \frac{\exp(c)}{\Lambda} \right) \mb{\mu}_k
    \end{split}
\end{equation}
Since the two factors are $<1$, we have that $\mb{w}_k^{\textrm{un}} = - (1- \gamma) \mb{\mu}_k$ for $\gamma<1$

\end{proof}

\begin{lemma} \label{lm:stat_pts}
    Let $\mb{z}, \mb{z'} \in \mathbb{R}^K$ be any two real vectors, and $\mb{y}_i, \mb{y}_k$ be two one-hot vectors corresponding to different classes. Consider the following constrained optimization problem:
    \[ \min_{\mb{z}, \mb{z'}} \mc{L}_{CE} (\mb{z}, \mb{y}_i) - \mc{L}_{CE} (\mb{z'}, \mb{y}_k) \quad s.t. \| \mb{z} \|_2^2 \leq 1, \| \mb{z'} \|_2^2 \leq 1, z_k = z'_k\]

    The KKT points of this objective satisfy the following conditions $z_j = z_l, \forall j,l \neq k,i$ and $z'_j = z'_l, \forall j,l \neq k$. 
\end{lemma}
\begin{proof}
    The lagrangian for our problem is:

    \begin{equation*}
        \begin{split}
            \mc{L} &= \log \left( \sum_{j=1}^K \exp(z_j) \right) - z_i + z'_k - \log \left( \sum_{j=1}^K \exp(z'_j) \right) \\  &+ \lambda_1 (\|\mb{z}\|_2^2 -1) +\lambda_2 (\|\mb{z'}\|_2^2 -1) + \lambda_3(z_k - z'_k)
        \end{split}
    \end{equation*}

    at stationary points of the Lagrangian, we have for entries of $\mb{z}$:
    \begin{equation*}
    \begin{split}        
        \frac{\partial \mc{L}}{\partial z_j} &= \frac{\exp{z_j}}{\sum_{l=1}^K \exp(z_l)} + \lambda_1 z_j = 0 \quad j \neq i,k \\
        \frac{\partial \mc{L}}{\partial z_i} &= \frac{\exp{z_i}}{\sum_{l=1}^K \exp(z_l)} -1 + \lambda_1 z_i = 0\\
        \frac{\partial \mc{L}}{\partial z_k} &= \frac{\exp{z_k}}{\sum_{l=1}^K \exp(z_l)} + (\lambda_1 - \lambda_3) z_k =0
    \end{split}
    \end{equation*}

    From the conditions on $z_j, j\neq k,i$ we obtain the following equation: $ \frac{\exp{z_j}}{- \lambda_1 \sum_{l=1}^K \exp(z_l)} = z_j$. Since the equation $c \exp(x) = x$ has only one solution in $x \in \mathbb{R}$, we get the condition that $z_j$ are all equal for $j\neq k,i$. The stationarity conditions for $z_i, z_k$ are different, and hence those values will be different.

    Using a similar argument for the stationarity conditions on $z'_j, j\neq k$ we show that $z'_j$ are all equal for $j \neq k$
\end{proof}

\newpage
\section{Tables on Experiments and Results}
\label{app:tables}

\subsection{Hyperparameter Settings}
\label{app:hyperparameter}
\begin{table*}[ht!]
\caption{Hyperparameters for ResNet-based experiments across CIFAR-10, CIFAR-100, and Tiny-ImageNet. 
SVD is training-free: no encoder updates; the batch size correspond only to data samples drawn separately from the retain and forget datasets.}
\label{tab:unlearn-hparams_resnet}
\centering
\resizebox{\textwidth}{!}{
\begin{tabular}{l l l c c c c l}
\toprule
Dataset & Model & Method & Epochs & Batch & LR & Mom. & Other Key Flags \\
\midrule

CIFAR-10 & ResNet18 & Original & 300 & 128 & 0.01 & 0.9 & cosine LR; WD=$5\times10^{-4}$ \\
CIFAR-100 & ResNet18 & Original & 300 & 128 & 0.01 & 0.9 & cosine LR; WD=$5\times10^{-4}$ \\
Tiny-ImageNet & ResNet50 & Original & 300 & 128 & 0.05 & 0.9 &cosine LR; WD=$5\times10^{-4}$\\

CIFAR-10 & ResNet18 & Retain-only Retrain & 200 & 128 & 0.01 & 0.9 & WD=$5\times10^{-4}$; val-ratio=0.1 \\
CIFAR-100 & ResNet18 &  Retain-only Retrain & 200 & 128 & 0.01 & 0.9 & WD=$5\times10^{-4}$ \\
Tiny-ImageNet & ResNet50 &  Retain-only Retrain & 150 & 256 & 0.05 & 0.9 & WD=$5\times10^{-4}$ \\

\midrule

CIFAR-10 & ResNet18 & Retain-only FT & 3 & 128 & $1\times10^{-3}$ & 0.9 & \\
CIFAR-100 & ResNet18 & Retain-only FT & 3 & 128 & $1\times10^{-3}$ & 0.9 & \\
Tiny-ImageNet & ResNet50 & Retain-only FT & 3 & 128 & $1\times10^{-3}$ & 0.9 &  \\

CIFAR-10 & ResNet18 & Random Label & 3 & 128 & $1\times10^{-4}$ & 0.9 &\\
CIFAR-100 & ResNet18 & Random Label & 3 & 128 & $3\times10^{-3}$ & -- & \\
Tiny-ImageNet & ResNet50 & Random Label & 3 & 128 & $5\times10^{-4}$ & -- &  \\

CIFAR-10 & ResNet18 & SalUN & 3 & 128 & $1\times10^{-4}$ & -- & threshold=0.5 \\
CIFAR-100 & ResNet18 & SalUN & 3 & 128 & $1\times10^{-3}$ & -- & threshold=0.5 \\
Tiny-ImageNet & ResNet50 & SalUN & 3 & 128 & $5\times10^{-4}$ & -- & threshold=0.5 \\

CIFAR-10 & ResNet18 & NegGrad+ & 3 & 128 & $1\times10^{-4}$ & -- & grad-clip=1.0 \\
CIFAR-100 & ResNet18 & NegGrad+ & 3 & 128 & $5\times10^{-3}$ & -- & grad-clip=1.0 \\
Tiny-ImageNet & ResNet50 & NegGrad+ & 3 & 128 & $5\times10^{-4}$ & -- & grad-clip=1.0 \\

CIFAR-10 & ResNet18 & SCRUB & 3 & 64 & $1\times10^{-4}$ & -- & sgda-bsz=64; msteps=2 \\
CIFAR-100 & ResNet18 & SCRUB & 3 & 64 & $1\times10^{-3}$ & -- & sgda-bsz=64; msteps=2 \\
Tiny-ImageNet & ResNet50 & SCRUB & 3 & 64 & $5\times10^{-3}$ & -- & sgda-bsz=64; msteps=2 \\

CIFAR-10 & ResNet18 & UNSIR & 3 & 128 & $5\times10^{-5}$ & -- & 3 epochs impair/repair training \\
CIFAR-100 & ResNet18 & UNSIR & 3 & 128 & $3\times10^{-5}$ & -- & 3 epochs impair/repair training \\
Tiny-ImageNet & ResNet50 & UNSIR & 3 & 128 & $2\times10^{-5}$ & -- & 3 epochs impair/repair training \\

CIFAR-10 & ResNet18 & SVD (TF) & -- & 900 & -- & -- & $\alpha_r = 1000,\alpha_f = 30$ \\
CIFAR-100 & ResNet18 & SVD (TF) & -- & 990 & -- & -- & $\alpha_r = 1000,\alpha_f = 30$  \\
Tiny-ImageNet & ResNet50 & SVD (TF) & --  & 999 & -- & -- & $\alpha_r = 30,\alpha_f = 10$  \\

\midrule

CIFAR-10 & ResNet18 & Random Label + CMF & 4 & 128 & $2\times10^{-3}$ & -- &  \\
CIFAR-100 & ResNet18 & Random Label + CMF & 4 & 128 & $2\times10^{-3}$ & -- &  \\
Tiny-ImageNet & ResNet50 & Random Label + CMF & 4 & 128 & $1\times10^{-2}$ & -- &  \\

CIFAR-10 & ResNet18 & SalUN + CMF & 4 & 128 & $2\times10^{-3}$ & -- & threshold=0.5 \\
CIFAR-100 & ResNet18 & SalUN + CMF & 4 & 128 & $2\times10^{-3}$ & -- & threshold=0.5 \\
Tiny-ImageNet & ResNet50 & SalUN + CMF & 4 & 128 & $1\times10^{-2}$ & -- & threshold=0.5 \\

CIFAR-10 & ResNet18 & NegGrad+ + CMF & 3 & 128 & $1\times10^{-4}$ & -- & grad-clip=1.0 \\
CIFAR-100 & ResNet18 & NegGrad+ + CMF & 3 & 128 & $1\times10^{-4}$ & -- & grad-clip=1.0 \\
Tiny-ImageNet & ResNet50 & NegGrad+ + CMF & 3 & 128 & $3\times10^{-5}$ & -- & grad-clip=1.0 \\

CIFAR-10 & ResNet18 & SCRUB + CMF & 3 & 64 & $5\times10^{-3}$ & -- & sgda-bsz=64; msteps=2\\
CIFAR-100 & ResNet18 & SCRUB + CMF & 3 & 64 & $5\times10^{-3}$ & -- & sgda-bsz=64; msteps=2 \\
Tiny-ImageNet & ResNet50 & SCRUB + CMF & 3 & 64 & $1\times10^{-3}$ & -- & sgda-bsz=64; msteps=2 \\

CIFAR-10 & ResNet18 & UNSIR + CMF & 3 & 128 & $5\times10^{-5}$ & -- & 3 epochs impair/repair training \\
CIFAR-100 & ResNet18 & UNSIR + CMF & 3 & 128 & $5\times10^{-5}$ & -- & 3 epochs impair/repair training \\
Tiny-ImageNet & ResNet50 & UNSIR + CMF & 3 & 128 & $2\times10^{-5}$ & -- & 3 epochs  impair/repair training \\

\bottomrule
\end{tabular}}
\label{tab:resnet-unlearn-hparams}
\end{table*}
\newpage
\begin{table*}[ht!]
\caption{Hyperparameters for ViT-based unlearning experiments across CIFAR-10, CIFAR-100, and Tiny-ImageNet.}
\centering
\resizebox{\textwidth}{!}{
\begin{tabular}{l l l c c c c l}
\toprule
Dataset & Model & Method & Epochs & Batch & LR & Mom. & Other Key Flags \\
\midrule

CIFAR-10 & ViT-S/16 & Original & 10 & 128 & $3\times10^{-4}$ & -- & pretrained backbone \\
CIFAR-100 & ViT-S/16 & Original & 10 & 128 & $3\times10^{-4}$ & -- & pretrained backbone \\
Tiny-ImageNet & ViT-S/16 & Original & 10 & 128 & $1\times10^{-4}$ & -- & pretrained backbone \\

CIFAR-10 & ViT-S/16 & Retrain & 10 & 128 & $3\times10^{-4}$ & -- & pretrained backbone \\
CIFAR-100 & ViT-S/16 & Retrain & 10 & 128 & $3\times10^{-4}$ & -- & pretrained backbone \\
Tiny-ImageNet & ViT-S/16 & Retrain & 10 & 128 & $1\times10^{-4}$ & -- & pretrained backbone \\

\midrule

CIFAR-10 & ViT-S/16 & Random Label & 3 & 128 & $3\times10^{-4}$ & -- &  \\
CIFAR-100 & ViT-S/16 & Random Label & 3 & 128 & $3\times10^{-4}$ & -- &  \\
Tiny-ImageNet & ViT-S/16 & Random Label & 10 & 128 & $1\times10^{-4}$ & -- &  \\

CIFAR-10 & ViT-S/16 & SalUN & 3 & 128 & $3\times10^{-4}$ & -- & threshold=0.5 \\
CIFAR-100 & ViT-S/16 & SalUN & 3 & 128 & $3\times10^{-4}$ & -- & threshold=0.5 \\
Tiny-ImageNet & ViT-S/16 & SalUN & 10 & 128 & $1\times10^{-4}$ & -- & threshold=0.5 \\

CIFAR-10 & ViT-S/16 & NegGrad+ & 3 & 128 & $3\times10^{-4}$ & -- & grad-clip=1.0 \\
CIFAR-100 & ViT-S/16 & NegGrad+ & 3 & 128 & $3\times10^{-4}$ & -- & grad-clip=1.0 \\
Tiny-ImageNet & ViT-S/16 & NegGrad+ & 3 & 128 & $1\times10^{-4}$ & -- & grad-clip=1.0 \\

\midrule

CIFAR-10 & ViT-S/16 & Random Label + CMF & 4 & 128 & $1\times10^{-3}$ & -- &  \\
CIFAR-100 & ViT-S/16 & Random Label + CMF & 4 & 128 & $2\times10^{-2}$ & -- &  \\
Tiny-ImageNet & ViT-S/16 & Random Label + CMF & 4 & 128 & $2\times10^{-3}$ & -- &  \\

CIFAR-10 & ViT-S/16 & SalUN + CMF & 4 & 128 & $3\times10^{-3}/2\times10^{-3}$ & -- & threshold=0.5 \\
CIFAR-100 & ViT-S/16 & SalUN + CMF & 4 & 128 & $2\times10^{-2}/1\times10^{-2}$ & -- & threshold=0.5 \\
Tiny-ImageNet & ViT-S/16 & SalUN + CMF & 4 & 128 & $1\times10^{-2}$ & -- & threshold=0.5 \\

CIFAR-10 & ViT-S/16 & NegGrad+ + CMF & 3 & 128 &  $5\times10^{-5}/5\times10^{-4}$ & -- & grad-clip=1.0 \\
CIFAR-100 & ViT-S/16 & NegGrad+ + CMF & 3 & 128 &  $5\times10^{-5}/5\times10^{-4}$ & -- & grad-clip=1.0 \\
Tiny-ImageNet & ViT-S/16 & NegGrad+ + CMF & 3 & 128 &  $1\times10^{-5}/1\times10^{-2}$ & -- & grad-clip=1.0 \\

\bottomrule
\end{tabular}}
\label{tab:unlearn_hparams_vit}
\end{table*}

\subsection{Experimental Results}

\begingroup
\setlength{\tabcolsep}{4pt}
\small

\begin{table*}[ht!]
\centering
\caption{
Evaluation of various MU methods on three datasets for unlearning certain number of classes. For all the (unlearned) models, we report both forget accuracy and retain accuracy evaluated for the entire model (labeled as Output) and the feature mapping by linear probe and nearest class center
(NCC) classification accuracy. We report the variance of the results to complement the mean performance values shown in \Cref{tab:ce_linearprobe}.
}
\vspace{-.1in}
\label{tab:Resnet_Normal_with_var}
\resizebox{1.0\linewidth}{!}{
\begin{tabular}{l l cc cc cc cc cc cc}
\toprule
\multirow{3}{*}{\textbf{Method}}
  & \multirow{3}{*}{\textbf{Accuracy}}
  & \multicolumn{4}{c}{\textbf{CIFAR-10}}
  & \multicolumn{4}{c}{\textbf{CIFAR-100}}
  & \multicolumn{4}{c}{\textbf{Tiny-ImageNet}}\\
\cmidrule(lr){3-6}\cmidrule(lr){7-10}\cmidrule(lr){11-14}
  &  & \multicolumn{2}{c}{1} & \multicolumn{2}{c}{3}
     & \multicolumn{2}{c}{1} & \multicolumn{2}{c}{10}
     & \multicolumn{2}{c}{1} & \multicolumn{2}{c}{20}\\
\cmidrule(lr){3-4}\cmidrule(lr){5-6}\cmidrule(lr){7-8}\cmidrule(lr){9-10}\cmidrule(lr){11-12}\cmidrule(lr){13-14}
  &  & Retain & Forget & Retain & Forget & Retain & Forget & Retain & Forget & Retain & Forget & Retain & Forget \\
\midrule
Original
   & Output & $93.98{\scriptstyle \pm 0.39}$ & $93.98{\scriptstyle \pm 3.48}$ & $94.00{\scriptstyle \pm 0.89}$ & $93.94{\scriptstyle \pm 2.09}$ & $74.61{\scriptstyle \pm 0.13}$ & $74.40{\scriptstyle \pm 12.86}$ & $74.47{\scriptstyle \pm 0.72}$ & $75.88{\scriptstyle \pm 6.48}$  & $65.27{\scriptstyle \pm 0.06}$ & $58.80{\scriptstyle \pm 12.54}$ & $65.15{\scriptstyle \pm 0.47}$ & $66.02{\scriptstyle \pm 4.25}$ \\
 & Linear Probe & $94.02{\scriptstyle \pm 0.39}$ & $94.02{\scriptstyle \pm 3.54}$ & $94.03{\scriptstyle \pm 0.90}$ & $94.00{\scriptstyle \pm 2.10}$ & $74.53{\scriptstyle \pm 0.12}$ & $75.00{\scriptstyle \pm 11.64}$ & $74.38{\scriptstyle \pm 0.73}$ & $75.90{\scriptstyle \pm 6.60}$ & $65.10{\scriptstyle \pm 0.05}$ & $60.80{\scriptstyle \pm 9.01}$ & $64.97{\scriptstyle \pm 0.44}$ & $66.08{\scriptstyle \pm 3.96}$ \\
 & NCC & $94.00{\scriptstyle \pm 0.39}$ & $93.99{\scriptstyle \pm 3.51}$ & $94.03{\scriptstyle \pm 0.93}$ & $93.92{\scriptstyle \pm 2.17}$ & $74.40{\scriptstyle \pm 0.11}$ & $75.00{\scriptstyle \pm 11.60}$ & $74.28{\scriptstyle \pm 0.72}$ & $75.68{\scriptstyle \pm 6.44}$ & $64.65{\scriptstyle \pm 0.04}$ & $60.00{\scriptstyle \pm 8.25}$ & $64.56{\scriptstyle \pm 0.41}$ & $65.00{\scriptstyle \pm 3.70}$ \\
\midrule
Retain-only Retrain 
  & Output & $94.74{\scriptstyle \pm 0.53}$ & $0.00{\scriptstyle \pm 0.00}$ & $95.37{\scriptstyle \pm 1.02}$ & $0.00{\scriptstyle \pm 0.00}$ & $76.01{\scriptstyle \pm 0.14}$ & $0.00{\scriptstyle \pm 0.00}$ & $76.50{\scriptstyle \pm 0.63}$ & $0.00{\scriptstyle \pm 0.00}$ & $66.52{\scriptstyle \pm 0.18}$ & $0.00{\scriptstyle \pm 0.00}$ & $66.38{\scriptstyle \pm 0.61}$ & $0.00{\scriptstyle \pm 0.00}$ \\
 & Linear Probe & $90.49{\scriptstyle \pm 1.00}$ & $77.35{\scriptstyle \pm 4.55}$ & $85.64{\scriptstyle \pm 2.69}$ & $67.33{\scriptstyle \pm 7.48}$ & $74.09{\scriptstyle \pm 1.06}$ & $85.20{\scriptstyle \pm 5.89}$ & $69.34{\scriptstyle \pm 1.42}$ & $60.94{\scriptstyle \pm 5.76}$ & $65.90{\scriptstyle \pm 0.16}$ & $46.40{\scriptstyle \pm 16.88}$ & $65.21{\scriptstyle \pm 0.48}$ & $30.36{\scriptstyle \pm 1.72}$ \\
 & NCC & $93.31{\scriptstyle \pm 0.45}$ & $47.06{\scriptstyle \pm 3.68}$ & $91.37{\scriptstyle \pm 1.86}$ & $37.07{\scriptstyle \pm 8.71}$ & $73.90{\scriptstyle \pm 1.28}$ & $70.40{\scriptstyle \pm 8.73}$ & $70.98{\scriptstyle \pm 0.68}$ & $43.18{\scriptstyle \pm 6.33}$ & $63.57{\scriptstyle \pm 1.10}$ & $71.20{\scriptstyle \pm 2.68}$ & $59.37{\scriptstyle \pm 0.56}$ & $44.30{\scriptstyle \pm 2.52}$ \\

\midrule
Retain-only FT
  & Output         & $94.26{\scriptstyle \pm 0.57}$ & $47.67{\scriptstyle \pm 29.14}$ & $95.24{\scriptstyle \pm 1.10}$ & $52.48{\scriptstyle \pm 20.15}$ & $74.08{\scriptstyle \pm 0.19}$ & $53.20{\scriptstyle \pm 18.75}$ & $74.53{\scriptstyle \pm 0.99}$ & $64.96{\scriptstyle \pm 7.84}$ & $65.26{\scriptstyle \pm 0.09}$ & $37.60{\scriptstyle \pm 9.84}$ & $65.60{\scriptstyle \pm 0.63}$ & $50.92{\scriptstyle \pm 6.92}$ \\
  & Linear Probe   & $93.94{\scriptstyle \pm 0.32}$ & $89.71{\scriptstyle \pm 5.08}$ & $94.14{\scriptstyle \pm 0.81}$ & $90.44{\scriptstyle \pm 2.40}$ & $73.89{\scriptstyle \pm 0.16}$ & $73.80{\scriptstyle \pm 10.76}$ & $73.97{\scriptstyle \pm 0.74}$ & $74.50{\scriptstyle \pm 6.31}$ & $64.44{\scriptstyle \pm 0.11}$ & $56.00{\scriptstyle \pm 9.80}$ & $64.05{\scriptstyle \pm 0.44}$ & $63.82{\scriptstyle \pm 2.80}$ \\
  & NCC   & $93.70{\scriptstyle \pm 0.25}$ & $89.63{\scriptstyle \pm 4.27}$ & $93.80{\scriptstyle \pm 0.87}$ & $88.75{\scriptstyle \pm 2.69}$ & $74.04{\scriptstyle \pm 0.16}$ & $73.20{\scriptstyle \pm 10.38}$ & $74.08{\scriptstyle \pm 0.79}$ & $73.56{\scriptstyle \pm 6.18}$ & $64.00{\scriptstyle \pm 0.14}$ & $57.60{\scriptstyle \pm 8.17}$ & $63.90{\scriptstyle \pm 0.29}$ & $63.58{\scriptstyle \pm 2.98}$ \\
  
\midrule
NegGrad+
  & Output        & $92.85{\scriptstyle \pm 1.20}$ &  $0.00{\scriptstyle \pm 0.00}$ & $93.29{\scriptstyle \pm 0.84}$ & $0.01{\scriptstyle \pm 0.01}$ & $69.90{\scriptstyle \pm 1.53}$ & $0.00{\scriptstyle \pm 0.00}$ & $70.80{\scriptstyle \pm 1.03}$ & $0.28{\scriptstyle \pm 0.22}$ &  $57.96{\scriptstyle \pm 0.89}$ & $0.00{\scriptstyle \pm 0.00}$ & $59.06{\scriptstyle \pm 2.62}$ & $0.00{\scriptstyle \pm 0.00}$ \\
  & Linear Probe   & $92.14{\scriptstyle \pm 0.46}$ & $67.09{\scriptstyle \pm 4.88}$ & $88.18{\scriptstyle \pm 1.50}$ & $73.91{\scriptstyle \pm 5.55}$ & $72.55{\scriptstyle \pm 0.44}$ & $67.20{\scriptstyle \pm 13.55}$ & $72.20{\scriptstyle \pm 0.51}$ & $62.32{\scriptstyle \pm 7.49}$ & $60.43{\scriptstyle \pm 0.40}$ & $58.00{\scriptstyle \pm 9.17}$ & $60.75{\scriptstyle \pm 0.44}$ & $54.68{\scriptstyle \pm 2.33}$ \\
  & NCC   & $91.33{\scriptstyle \pm 0.62}$ & $52.00{\scriptstyle \pm 3.45}$ & $87.28{\scriptstyle \pm 2.10}$ & $59.17{\scriptstyle \pm 7.68}$ & $71.58{\scriptstyle \pm 0.81}$ & $41.40{\scriptstyle \pm 6.66}$ & $70.08{\scriptstyle \pm 0.80}$ & $41.24{\scriptstyle \pm 4.58}$ & $59.01{\scriptstyle \pm 0.43}$ & $37.60{\scriptstyle \pm 8.17}$ & $56.45{\scriptstyle \pm 1.21}$ & $36.78{\scriptstyle \pm 3.98}$ \\
  
\midrule
SVD
  & Output        & $92.02{\scriptstyle \pm 1.47}$ &  $0.00{\scriptstyle \pm 0.00}$ & $94.05{\scriptstyle \pm 1.02}$ & $57.43{\scriptstyle \pm 1.35}$ &$71.09{\scriptstyle \pm 0.60}$ &  $0.00{\scriptstyle \pm 0.00}$ & $73.10{\scriptstyle \pm 1.23}$ & $55.56{\scriptstyle \pm 4.46}$ & $64.43{\scriptstyle \pm 0.19}$ & $2.00{\scriptstyle \pm 1.41}$ & $65.45{\scriptstyle \pm 0.45}$ & $59.02{\scriptstyle \pm 4.38}$ \\
  & Linear Probe   & $90.44{\scriptstyle \pm 1.24}$ & $61.80{\scriptstyle \pm 4.21}$ & $92.43{\scriptstyle \pm 1.65}$ & $83.58{\scriptstyle \pm 2.23}$ & $73.14{\scriptstyle \pm 0.37}$ & $67.00{\scriptstyle \pm 10.26}$ & $73.38{\scriptstyle \pm 0.84}$ & $74.08{\scriptstyle \pm 5.37}$ & $63.10{\scriptstyle \pm 0.06}$ & $60.40{\scriptstyle \pm 11.78}$ & $63.03{\scriptstyle \pm 0.52}$ & $63.64{\scriptstyle \pm 4.66}$ \\
  & NCC   & $90.11{\scriptstyle \pm 1.64}$ & $34.54{\scriptstyle \pm 3.88}$ & $93.23{\scriptstyle \pm 1.47}$ & $72.32{\scriptstyle \pm 1.62}$ & $71.81{\scriptstyle \pm 0.58}$ & $64.80{\scriptstyle \pm 9.09}$ & $73.33{\scriptstyle \pm 0.97}$ & $72.22{\scriptstyle \pm 5.38}$  & $64.65{\scriptstyle \pm 0.04}$ & $60.00{\scriptstyle \pm 8.25}$ & $64.56{\scriptstyle \pm 0.41}$ & $65.00{\scriptstyle \pm 3.70}$  \\

\midrule
Random-label 
  & Output        & $92.93{\scriptstyle \pm 0.94}$ &  $0.00{\scriptstyle \pm 0.00}$ & $94.14{\scriptstyle \pm 1.06}$ &  $0.00{\scriptstyle \pm 0.00}$ & $72.33{\scriptstyle \pm 0.15}$ & $0.00{\scriptstyle \pm 0.00}$ & $72.19{\scriptstyle \pm 0.19}$ & $0.00{\scriptstyle \pm 0.00}$ & $65.48{\scriptstyle \pm 0.05}$ & $0.40{\scriptstyle \pm 0.89}$ & $64.85{\scriptstyle \pm 0.47}$ & $0.98{\scriptstyle \pm 0.33}$ \\
  & Linear Probe   & $92.65{\scriptstyle \pm 0.58}$ & $92.49{\scriptstyle \pm 3.70}$ & $92.45{\scriptstyle \pm 1.00}$ & $90.25{\scriptstyle \pm 3.07}$ & $73.08{\scriptstyle \pm 0.27}$ & $79.00{\scriptstyle \pm 11.60}$ & $72.08{\scriptstyle \pm 0.59}$ & $72.08{\scriptstyle \pm 7.13}$ & $64.07{\scriptstyle \pm 0.19}$ & $58.00{\scriptstyle \pm 8.49}$ & $62.02{\scriptstyle \pm 0.48}$ & $57.82{\scriptstyle \pm 2.86}$\\
  & NCC   & $92.25{\scriptstyle \pm 1.09}$ & $80.25{\scriptstyle \pm 11.82}$ & $91.92{\scriptstyle \pm 1.20}$ & $73.22{\scriptstyle \pm 11.14}$ & $72.38{\scriptstyle \pm 0.51}$ & $87.20{\scriptstyle \pm 8.67}$ & $70.67{\scriptstyle \pm 0.65}$ & $62.22{\scriptstyle \pm 7.95}$& $63.20{\scriptstyle \pm 0.28}$ & $69.20{\scriptstyle \pm 7.69}$ & $59.42{\scriptstyle \pm 0.55}$ & $42.72{\scriptstyle \pm 2.57}$ \\

\midrule
SalUn
  & Output       & $93.19{\scriptstyle \pm 0.79}$ &  $0.00{\scriptstyle \pm 0.00}$ & $94.43{\scriptstyle \pm 0.96}$ &  $0.00{\scriptstyle \pm 0.00}$ & $72.96{\scriptstyle \pm 0.20}$ & $0.00{\scriptstyle \pm 0.00}$ & $72.92{\scriptstyle \pm 0.84}$ & $0.06{\scriptstyle \pm 0.09}$ & $65.49{\scriptstyle \pm 0.04}$ & $0.40{\scriptstyle \pm 0.89}$ & $64.63{\scriptstyle \pm 0.31}$ & $3.40{\scriptstyle \pm 0.83}$ \\
  & Linear Probe   & $93.05{\scriptstyle \pm 0.42}$ & $92.57{\scriptstyle \pm 3.23}$ & $92.89{\scriptstyle \pm 0.81}$ & $89.63{\scriptstyle \pm 3.78}$ & $73.26{\scriptstyle \pm 0.12}$ & $77.80{\scriptstyle \pm 16.12}$ & $72.10{\scriptstyle \pm 0.70}$ & $72.66{\scriptstyle \pm 6.93}$ & $64.07{\scriptstyle \pm 0.07}$ & $55.60{\scriptstyle \pm 9.53}$ & $61.86{\scriptstyle \pm 0.28}$ & $56.52{\scriptstyle \pm 2.54}$  \\

  & NCC   & $91.31{\scriptstyle \pm 0.42}$ & $93.70{\scriptstyle \pm 2.14}$ & $91.99{\scriptstyle \pm 0.98}$ & $68.45{\scriptstyle \pm 7.29}$ & $72.65{\scriptstyle \pm 0.30}$ & $85.60{\scriptstyle \pm 10.64}$ & $70.62{\scriptstyle \pm 0.83}$ & $63.34{\scriptstyle \pm 6.14}$ & $63.33{\scriptstyle \pm 0.20}$ & $68.80{\scriptstyle \pm 6.57}$ & $59.07{\scriptstyle \pm 0.57}$ & $39.70{\scriptstyle \pm 3.48}$   \\
\midrule
SCRUB
  & Output        & $91.37{\scriptstyle \pm 2.25}$ & $0.00{\scriptstyle \pm 0.00}$ & $93.61{\scriptstyle \pm 1.25}$ &  $0.00{\scriptstyle \pm 0.00}$ & $73.67{\scriptstyle \pm 0.17}$ & $0.20{\scriptstyle \pm 0.45}$ & $74.56{\scriptstyle \pm 0.83}$ & $0.32{\scriptstyle \pm 0.32}$ & $65.43{\scriptstyle \pm 0.11}$ & $1.20{\scriptstyle \pm 1.79}$ & $65.30{\scriptstyle \pm 0.46}$ & $5.48{\scriptstyle \pm 3.46}$  \\
  & Linear Probe   & $91.71{\scriptstyle \pm 1.76}$ & $74.79{\scriptstyle \pm 8.03}$ & $91.88{\scriptstyle \pm 1.11}$ & $78.23{\scriptstyle \pm 2.99}$ &  $73.69{\scriptstyle \pm 0.13}$ & $72.40{\scriptstyle \pm 13.20}$ & $73.05{\scriptstyle \pm 0.72}$ & $66.92{\scriptstyle \pm 6.67}$ & $64.44{\scriptstyle \pm 0.34}$ & $56.00{\scriptstyle \pm 9.49}$ & $62.87{\scriptstyle \pm 0.50}$ & $56.90{\scriptstyle \pm 4.63}$ \\

  & NCC   & $89.96{\scriptstyle \pm 1.60}$ & $55.30{\scriptstyle \pm 10.92}$ & $89.73{\scriptstyle \pm 1.85}$ & $47.52{\scriptstyle \pm 6.11}$ & $73.51{\scriptstyle \pm 0.19}$ & $72.40{\scriptstyle \pm 10.76}$ & $72.35{\scriptstyle \pm 0.54}$ & $53.58{\scriptstyle \pm 6.04}$ & $64.37{\scriptstyle \pm 0.08}$ & $60.40{\scriptstyle \pm 6.07}$ & $61.41{\scriptstyle \pm 0.41}$ & $53.50{\scriptstyle \pm 4.93}$ \\

  \midrule
\multirow{3}{*}{UNSIR} & Output & $91.84{\scriptstyle \pm 0.75}$ & $0.48{\scriptstyle \pm 0.58}$ & $92.87{\scriptstyle \pm 1.33}$ & $0.01{\scriptstyle \pm 0.01}$ & $73.77{\scriptstyle \pm 0.28}$ & $3.80{\scriptstyle \pm 3.56}$ & $73.58{\scriptstyle \pm 0.85}$ & $14.58{\scriptstyle \pm 4.89}$ & $64.66{\scriptstyle \pm 0.27}$ & $0.00{\scriptstyle \pm 0.00}$ & $65.46{\scriptstyle \pm 0.61}$ & $9.72{\scriptstyle \pm 5.11}$ \\
 & Linear Probe & $89.71{\scriptstyle \pm 0.50}$ & $85.29{\scriptstyle \pm 7.14}$ & $88.21{\scriptstyle \pm 1.97}$ & $70.59{\scriptstyle \pm 6.63}$ & $73.40{\scriptstyle \pm 0.27}$ & $73.60{\scriptstyle \pm 14.88}$ & $72.15{\scriptstyle \pm 0.78}$ & $67.64{\scriptstyle \pm 6.85}$ & $63.88{\scriptstyle \pm 0.23}$ & $61.20{\scriptstyle \pm 9.96}$ & $62.87{\scriptstyle \pm 0.33}$ & $61.10{\scriptstyle \pm 3.24}$ \\
 & NCC & $89.73{\scriptstyle \pm 0.49}$ & $62.72{\scriptstyle \pm 5.04}$ & $87.73{\scriptstyle \pm 2.37}$ & $49.42{\scriptstyle \pm 8.20}$ & $73.05{\scriptstyle \pm 0.40}$ & $71.00{\scriptstyle \pm 13.32}$ & $71.26{\scriptstyle \pm 0.55}$ & $59.24{\scriptstyle \pm 6.82}$ & $63.16{\scriptstyle \pm 0.16}$ & $59.20{\scriptstyle \pm 5.02}$ & $61.80{\scriptstyle \pm 0.32}$ & $56.14{\scriptstyle \pm 2.40}$ \\
\bottomrule
\end{tabular}
}
\vspace{-.2in}
\end{table*}

\endgroup

\begin{table*}[ht!]
\centering
\caption{ 
Unlearning from full model VS only classifier evaluated through output-level forget and retain accuracies. We report the variance of the results to complement the mean performance values shown in \Cref{tab:unlearn_lastlayer}.
}
\label{tab:unlearn_lastlayer_with_var}
\vspace{-.15in}
\resizebox{1.0\linewidth}{!}{
\begin{tabular}{l l cc cc cc cc}
\toprule
\multirow{3}{*}{\textbf{Method}}
  & \multirow{3}{*}{\shortstack{\textbf{Layers}\\\textbf{Finetuned}}}
  & \multicolumn{4}{c}{\textbf{CIFAR-10}}
  & \multicolumn{4}{c}{\textbf{CIFAR-100}} \\
\cmidrule(lr){3-6}\cmidrule(lr){7-10}
  &  & \multicolumn{2}{c}{1} & \multicolumn{2}{c}{3}
     & \multicolumn{2}{c}{1} & \multicolumn{2}{c}{10} \\
\cmidrule(lr){3-4}\cmidrule(lr){5-6}\cmidrule(lr){7-8}\cmidrule(lr){9-10}
  &  & Retain & Forget & Retain & Forget & Retain & Forget & Retain & Forget \\
\midrule
Original
  & Full Model             & $93.98{\scriptstyle \pm 0.39}$ & $93.98{\scriptstyle \pm 3.48}$ & $94.00{\scriptstyle \pm 0.89}$ & $93.94{\scriptstyle \pm 2.09}$ & $74.61{\scriptstyle \pm 0.13}$ & $74.40{\scriptstyle \pm 12.86}$ & $74.47{\scriptstyle \pm 0.72}$ & $75.88{\scriptstyle \pm 6.48}$ \\

\midrule
Retain-only Retrain
  & Full Model             & $94.74{\scriptstyle \pm 0.53}$ & $0.00{\scriptstyle \pm 0.00}$ & $95.37{\scriptstyle \pm 1.02}$ & $0.00{\scriptstyle \pm 0.00}$ & $76.01{\scriptstyle \pm 0.14}$ & $0.00{\scriptstyle \pm 0.00}$ & $76.50{\scriptstyle \pm 0.63}$ & $0.00{\scriptstyle \pm 0.00}$  \\

\midrule
\multirow{2}{*}{Retain-only FT}
  & Full Model             & $94.26{\scriptstyle \pm 0.57}$ & $47.67{\scriptstyle \pm 29.14}$ & $95.24{\scriptstyle \pm 1.10}$ & $52.48{\scriptstyle \pm 20.15}$ & $74.08{\scriptstyle \pm 0.19}$ & $53.20{\scriptstyle \pm 18.75}$ & $74.53{\scriptstyle \pm 0.99}$ & $64.96{\scriptstyle \pm 7.84}$\\
  & Classifier only        & $93.20{\scriptstyle \pm 0.72}$ &  $0.00{\scriptstyle \pm 0.00}$ & $93.66{\scriptstyle \pm 0.85}$ &  $0.00{\scriptstyle \pm 0.00}$ & $73.29{\scriptstyle \pm 0.37}$ &  $0.00{\scriptstyle \pm 0.00}$ & $73.51{\scriptstyle \pm 0.84}$ &  $0.00{\scriptstyle \pm 0.00}$ \\

\midrule
\multirow{2}{*}{NegGrad+}
  & Full Model             & $92.85{\scriptstyle \pm 1.20}$ &  $0.00{\scriptstyle \pm 0.00}$ & $93.29{\scriptstyle \pm 0.84}$ & $0.01{\scriptstyle \pm 0.01}$ & $69.90{\scriptstyle \pm 1.53}$ & $0.00{\scriptstyle \pm 0.00}$ & $70.80{\scriptstyle \pm 1.03}$ & $0.28{\scriptstyle \pm 0.22}$  \\
  & Classifier only        & $93.08{\scriptstyle \pm 0.85}$ &  $0.00{\scriptstyle \pm 0.00}$ & $93.72{\scriptstyle \pm 0.97}$ &  $0.00{\scriptstyle \pm 0.00}$ & $73.28{\scriptstyle \pm 1.11}$ &  $0.00{\scriptstyle \pm 0.00}$ & $66.85{\scriptstyle \pm 6.25}$ &  $0.00{\scriptstyle \pm 0.00}$ \\

\midrule
\multirow{2}{*}{Random-label}
  & Full Model             & $92.93{\scriptstyle \pm 0.94}$ &  $0.00{\scriptstyle \pm 0.00}$ & $94.14{\scriptstyle \pm 1.06}$ &  $0.00{\scriptstyle \pm 0.00}$ & $72.33{\scriptstyle \pm 0.15}$ & $0.00{\scriptstyle \pm 0.00}$ & $72.19{\scriptstyle \pm 0.19}$ & $0.00{\scriptstyle \pm 0.00}$ \\
  & Classifier only        & $93.39{\scriptstyle \pm 0.75}$ &  $0.00{\scriptstyle \pm 0.00}$ & $94.56{\scriptstyle \pm 1.04}$ &  $0.00{\scriptstyle \pm 0.00}$ & $73.87{\scriptstyle \pm 0.12}$ &  $0.20{\scriptstyle \pm 0.40}$ & $74.09{\scriptstyle \pm 0.73}$ &  $2.32{\scriptstyle \pm 0.56}$ \\
\midrule
\multirow{2}{*}{Salun}
  & Full Model             & $93.19{\scriptstyle \pm 0.79}$ &  $0.00{\scriptstyle \pm 0.00}$ & $94.43{\scriptstyle \pm 0.96}$ &  $0.00{\scriptstyle \pm 0.00}$ & $72.96{\scriptstyle \pm 0.20}$ & $0.00{\scriptstyle \pm 0.00}$ & $72.92{\scriptstyle \pm 0.84}$ & $0.06{\scriptstyle \pm 0.09}$ \\
  & Classifier only        & $93.38{\scriptstyle \pm 0.72}$ &  $0.00{\scriptstyle \pm 0.00}$ & $94.44{\scriptstyle \pm 1.07}$ &  $0.00{\scriptstyle \pm 0.00}$ & $73.93{\scriptstyle \pm 0.20}$ &  $1.00{\scriptstyle \pm 0.63}$ & $73.98{\scriptstyle \pm 0.72}$ &  $4.18{\scriptstyle \pm 1.20}$ \\
\midrule
\multirow{2}{*}{SVD}
  & Full Model             & $92.02{\scriptstyle \pm 1.47}$ &  $0.00{\scriptstyle \pm 0.00}$ & $94.05{\scriptstyle \pm 1.02}$ & $57.43{\scriptstyle \pm 1.35}$ &$71.09{\scriptstyle \pm 0.60}$ &  $0.00{\scriptstyle \pm 0.00}$ & $73.10{\scriptstyle \pm 1.23}$ & $55.56{\scriptstyle \pm 4.46}$ \\
  & Classifier only        & $93.53{\scriptstyle \pm 0.66}$ &  $0.01{\scriptstyle \pm 0.03}$ & $93.73{\scriptstyle \pm 0.87}$ & $68.45{\scriptstyle \pm 9.33}$ & $73.47{\scriptstyle \pm 0.18}$ &  $2.60{\scriptstyle \pm 1.86}$ & $74.13{\scriptstyle \pm 0.66}$ & $50.56{\scriptstyle \pm 5.57}$ \\
  \midrule
\multirow{2}{*}{SCRUB}
  & Full Model            & $91.37{\scriptstyle \pm 2.25}$ & $0.00{\scriptstyle \pm 0.00}$ & $93.61{\scriptstyle \pm 1.25}$ &  $0.00{\scriptstyle \pm 0.00}$ & $73.67{\scriptstyle \pm 0.17}$ & $0.20{\scriptstyle \pm 0.45}$ & $74.56{\scriptstyle \pm 0.83}$ & $0.32{\scriptstyle \pm 0.32}$ \\
  & Classifier only        & $93.12{\scriptstyle \pm 0.76}$ &  $0.00{\scriptstyle \pm 0.00}$ & $93.68{\scriptstyle \pm 0.95}$ & $0.00{\scriptstyle \pm 0.00}$ & $74.01{\scriptstyle \pm 0.24}$ &  $5.00{\scriptstyle \pm 9.01}$ & $74.01{\scriptstyle \pm 0.75}$ & $0.00{\scriptstyle \pm 0.00}$ \\
  \midrule
\multirow{2}{*}{UNSIR}
  & Full Model             & $91.84{\scriptstyle \pm 0.75}$ & $0.48{\scriptstyle \pm 0.58}$ & $92.87{\scriptstyle \pm 1.33}$ & $0.01{\scriptstyle \pm 0.01}$ & $73.77{\scriptstyle \pm 0.28}$ & $3.80{\scriptstyle \pm 3.56}$ & $73.58{\scriptstyle \pm 0.85}$ & $14.58{\scriptstyle \pm 4.89}$  \\
  & Classifier only        & $94.48{\scriptstyle \pm 0.66}$ & $0.00{\scriptstyle \pm 0.00}$ & $94.46{\scriptstyle \pm 0.82}$ & $1.39{\scriptstyle \pm 0.96}$ & $74.64{\scriptstyle \pm 0.13}$ & $0.00{\scriptstyle \pm 0.00}$ & $74.27{\scriptstyle \pm 0.79}$ & $0.52{\scriptstyle \pm 1.11}$  \\
\bottomrule
\end{tabular}
}
\vspace{-.1in}
\end{table*}

\begingroup
\setlength{\tabcolsep}{4pt}
\small
\begin{table*}[t]
\centering
\caption{Evaluation of MU methods on \textbf{ResNet} with CMF classifiers on three datasets for unlearning certain number of classes. We report the variance of the results to complement the mean performance values shown in \Cref{tab:eval_resnet18}.}
\label{tab:eval_resnet18_with_var}
\vspace{-0.12in}
\resizebox{1.0\linewidth}{!}{
\begin{tabular}{l l cccc cccc cccc}
\toprule
\multirow{3}{*}{\textbf{Method}} & \multirow{3}{*}{\textbf{Accuracy}} & \multicolumn{4}{c}{\textbf{CIFAR-10}} & \multicolumn{4}{c}{\textbf{CIFAR-100}} & \multicolumn{4}{c}{\textbf{Tiny-ImageNet}} \\
\cmidrule(lr){3-6}\cmidrule(lr){7-10}\cmidrule(lr){11-14}
 &  & \multicolumn{2}{c}{\textbf{1}} & \multicolumn{2}{c}{\textbf{3}} & \multicolumn{2}{c}{\textbf{1}} & \multicolumn{2}{c}{\textbf{10}} & \multicolumn{2}{c}{\textbf{1}} & \multicolumn{2}{c}{\textbf{20}}\\
\cmidrule(lr){3-4}\cmidrule(lr){5-6}\cmidrule(lr){7-8}\cmidrule(lr){9-10}\cmidrule(lr){11-12}\cmidrule(lr){13-14}
 &  & \textbf{Retain} & \textbf{Forget} & \textbf{Retain} & \textbf{Forget} & \textbf{Retain} & \textbf{Forget} & \textbf{Retain} & \textbf{Forget}& \textbf{Retain} & \textbf{Forget} & \textbf{Retain} & \textbf{Forget} \\
 \midrule
 \multirow{3}{*}{Original} & Output & $93.98{\scriptstyle \pm 0.39}$ & $93.98{\scriptstyle \pm 3.48}$ & $94.00{\scriptstyle \pm 0.89}$ & $93.94{\scriptstyle \pm 2.09}$ & $74.61{\scriptstyle \pm 0.13}$ & $74.40{\scriptstyle \pm 12.86}$ & $74.47{\scriptstyle \pm 0.72}$ & $75.88{\scriptstyle \pm 6.48}$ & $65.27{\scriptstyle \pm 0.06}$ & $58.80{\scriptstyle \pm 12.54}$ & $65.15{\scriptstyle \pm 0.47}$ & $66.02{\scriptstyle \pm 4.25}$ \\
 & Linear Probe & $94.02{\scriptstyle \pm 0.39}$ & $94.02{\scriptstyle \pm 3.54}$ & $94.03{\scriptstyle \pm 0.90}$ & $94.00{\scriptstyle \pm 2.10}$ & $74.53{\scriptstyle \pm 0.12}$ & $75.00{\scriptstyle \pm 11.64}$ & $74.38{\scriptstyle \pm 0.73}$ & $75.90{\scriptstyle \pm 6.60}$ & $65.10{\scriptstyle \pm 0.05}$ & $60.80{\scriptstyle \pm 9.01}$ & $64.97{\scriptstyle \pm 0.44}$ & $66.08{\scriptstyle \pm 3.96}$ \\
 & NCC & $94.00{\scriptstyle \pm 0.39}$ & $93.99{\scriptstyle \pm 3.51}$ & $94.03{\scriptstyle \pm 0.93}$ & $93.92{\scriptstyle \pm 2.17}$ & $74.40{\scriptstyle \pm 0.11}$ & $75.00{\scriptstyle \pm 11.60}$ & $74.28{\scriptstyle \pm 0.72}$ & $75.68{\scriptstyle \pm 6.44}$ & $64.65{\scriptstyle \pm 0.04}$ & $60.00{\scriptstyle \pm 8.25}$ & $64.56{\scriptstyle \pm 0.41}$ & $65.00{\scriptstyle \pm 3.70}$ \\
 \midrule
 \multirow{3}{*}{Retain-only Retrain} & Output & $94.74{\scriptstyle \pm 0.53}$ & $0.00{\scriptstyle \pm 0.00}$ & $95.37{\scriptstyle \pm 1.02}$ & $0.00{\scriptstyle \pm 0.00}$ & $76.01{\scriptstyle \pm 0.14}$ & $0.00{\scriptstyle \pm 0.00}$ & $76.50{\scriptstyle \pm 0.63}$ & $0.00{\scriptstyle \pm 0.00}$ & $66.52{\scriptstyle \pm 0.18}$ & $0.00{\scriptstyle \pm 0.00}$ & $66.38{\scriptstyle \pm 0.61}$ & $0.00{\scriptstyle \pm 0.00}$ \\
 & Linear Probe& $90.49{\scriptstyle \pm 1.00}$ & $77.35{\scriptstyle \pm 4.55}$ & $85.64{\scriptstyle \pm 2.69}$ & $67.33{\scriptstyle \pm 7.48}$ & $74.09{\scriptstyle \pm 1.06}$ & $85.20{\scriptstyle \pm 5.89}$ & $69.34{\scriptstyle \pm 1.42}$ & $60.94{\scriptstyle \pm 5.76}$ & $65.90{\scriptstyle \pm 0.16}$ & $46.40{\scriptstyle \pm 16.88}$ & $65.21{\scriptstyle \pm 0.48}$ & $30.36{\scriptstyle \pm 1.72}$ \\
 & NCC & $93.31{\scriptstyle \pm 0.45}$ & $47.06{\scriptstyle \pm 3.68}$ & $91.37{\scriptstyle \pm 1.86}$ & $37.07{\scriptstyle \pm 8.71}$ & $73.90{\scriptstyle \pm 1.28}$ & $70.40{\scriptstyle \pm 8.73}$ & $70.98{\scriptstyle \pm 0.68}$ & $43.18{\scriptstyle \pm 6.33}$ & $63.57{\scriptstyle \pm 1.10}$ & $71.20{\scriptstyle \pm 2.68}$ & $59.37{\scriptstyle \pm 0.56}$ & $44.30{\scriptstyle \pm 2.52}$ \\
\midrule
\multirow{3}{*}{Random-label with CMF} & Output  & $94.27{\scriptstyle \pm 0.57}$ & $80.70{\scriptstyle \pm 8.73}$ & $94.81{\scriptstyle \pm 0.99}$ & $75.69{\scriptstyle \pm 4.88}$  & $74.38{\scriptstyle \pm 0.12}$ & $55.60{\scriptstyle \pm 16.35}$ & $74.85{\scriptstyle \pm 0.77}$ & $54.40{\scriptstyle \pm 7.52}$ & $62.01{\scriptstyle \pm 0.25}$ & $22.40{\scriptstyle \pm 14.66}$ & $62.25{\scriptstyle \pm 0.61}$ & $22.20{\scriptstyle \pm 1.76}$\\
 & Linear Probe & $94.19{\scriptstyle \pm 0.52}$ & $85.71{\scriptstyle \pm 5.82}$ & $94.49{\scriptstyle \pm 0.95}$ & $81.47{\scriptstyle \pm 3.78}$ & $74.63{\scriptstyle \pm 0.17}$ & $59.20{\scriptstyle \pm 12.68}$ & $74.98{\scriptstyle \pm 0.76}$ & $66.44{\scriptstyle \pm 6.87}$ & $62.56{\scriptstyle \pm 0.11}$ & $32.80{\scriptstyle \pm 9.96}$ & $62.38{\scriptstyle \pm 0.30}$ & $36.58{\scriptstyle \pm 2.56}$ \\
 & NCC & $94.25{\scriptstyle \pm 0.55}$ & $82.04{\scriptstyle \pm 8.34}$ & $94.73{\scriptstyle \pm 0.95}$ & $76.74{\scriptstyle \pm 4.65}$ & $74.49{\scriptstyle \pm 0.11}$ & $59.20{\scriptstyle \pm 15.42}$ & $74.82{\scriptstyle \pm 0.79}$ & $60.24{\scriptstyle \pm 7.04}$ & $61.96{\scriptstyle \pm 0.24}$ & $27.20{\scriptstyle \pm 11.80}$ & $61.93{\scriptstyle \pm 0.48}$ & $28.24{\scriptstyle \pm 2.40}$ \\
\midrule
\multirow{3}{*}{Salun with CMF} & Output& $94.33{\scriptstyle \pm 0.60}$ & $78.07{\scriptstyle \pm 10.38}$ & $95.01{\scriptstyle \pm 0.99}$ & $75.96{\scriptstyle \pm 4.06}$ & $74.62{\scriptstyle \pm 0.11}$ & $60.40{\scriptstyle \pm 14.31}$ & $74.98{\scriptstyle \pm 0.69}$ & $62.10{\scriptstyle \pm 6.97}$  
& $62.62{\scriptstyle \pm 0.08}$ & $30.00{\scriptstyle \pm 14.21}$ & $63.14{\scriptstyle \pm 0.42}$ & $32.74{\scriptstyle \pm 2.53}$ \\
 & Linear Probe  & $94.26{\scriptstyle \pm 0.58}$ & $84.68{\scriptstyle \pm 7.30}$ & $94.60{\scriptstyle \pm 1.01}$ & $83.39{\scriptstyle \pm 4.22}$ & $74.79{\scriptstyle \pm 0.14}$ & $59.40{\scriptstyle \pm 15.99}$ & $75.18{\scriptstyle \pm 0.72}$ & $67.20{\scriptstyle \pm 6.73}$ & $63.18{\scriptstyle \pm 0.15}$ & $41.20{\scriptstyle \pm 8.07}$ & $63.12{\scriptstyle \pm 0.20}$ & $46.14{\scriptstyle \pm 2.37}$ \\
 & NCC & $94.32{\scriptstyle \pm 0.60}$ & $79.50{\scriptstyle \pm 9.80}$ & $94.91{\scriptstyle \pm 0.99}$ & $77.31{\scriptstyle \pm 4.34}$ & $74.78{\scriptstyle \pm 0.05}$ & $63.60{\scriptstyle \pm 13.92}$ & $75.03{\scriptstyle \pm 0.65}$ & $65.20{\scriptstyle \pm 6.75}$ & $62.67{\scriptstyle \pm 0.10}$ & $37.20{\scriptstyle \pm 9.96}$ & $62.77{\scriptstyle \pm 0.31}$ & $38.10{\scriptstyle \pm 2.10}$ \\
\midrule
\multirow{3}{*}{NegGrad+ with CMF} & Output & $91.87{\scriptstyle \pm 1.15}$ & $54.50{\scriptstyle \pm 6.24}$ & $94.97{\scriptstyle \pm 0.87}$ & $60.00{\scriptstyle \pm 6.42}$ & $71.82{\scriptstyle \pm 1.26}$ & $41.00{\scriptstyle \pm 15.46}$ & $71.45{\scriptstyle \pm 0.70}$ & $30.38{\scriptstyle \pm 4.43}$ & $61.18{\scriptstyle \pm 0.67}$ & $22.00{\scriptstyle \pm 12.88}$ & $60.82{\scriptstyle \pm 1.04}$ & $41.42{\scriptstyle \pm 3.41}$\\
 & Linear Probe & $92.35{\scriptstyle \pm 0.94}$ & $68.02{\scriptstyle \pm 5.00}$ & $94.60{\scriptstyle \pm 0.93}$ & $75.29{\scriptstyle \pm 4.62}$ & $72.86{\scriptstyle \pm 0.72}$ & $55.20{\scriptstyle \pm 13.88}$ & $72.35{\scriptstyle \pm 0.70}$ & $50.62{\scriptstyle \pm 5.74}$ & $62.90{\scriptstyle \pm 0.45}$ & $41.20{\scriptstyle \pm 11.01}$ & $62.66{\scriptstyle \pm 0.92}$ & $53.82{\scriptstyle \pm 2.59}$ \\
 & NCC & $91.99{\scriptstyle \pm 1.10}$ & $57.83{\scriptstyle \pm 5.83}$ & $94.93{\scriptstyle \pm 0.87}$ & $62.89{\scriptstyle \pm 6.01}$ & $72.36{\scriptstyle \pm 1.11}$ & $51.20{\scriptstyle \pm 14.04}$ & $71.75{\scriptstyle \pm 0.72}$ & $41.54{\scriptstyle \pm 5.18}$ & $62.31{\scriptstyle \pm 0.40}$ & $38.80{\scriptstyle \pm 11.54}$ & $62.11{\scriptstyle \pm 0.98}$ & $51.30{\scriptstyle \pm 2.66}$ \\
 \midrule
\multirow{3}{*}{Scrub with CMF} & Output  & $92.51{\scriptstyle \pm 0.91}$ & $33.78{\scriptstyle \pm 6.39}$ & $95.37{\scriptstyle \pm 1.00}$ & $35.11{\scriptstyle \pm 7.27}$ & $73.86{\scriptstyle \pm 0.34}$ & $40.60{\scriptstyle \pm 18.53}$ & $74.27{\scriptstyle \pm 0.74}$ & $35.02{\scriptstyle \pm 7.13}$ & $61.64{\scriptstyle \pm 0.18}$ & $27.20{\scriptstyle \pm 13.54}$ & $63.31{\scriptstyle \pm 0.35}$ & $47.76{\scriptstyle \pm 4.62}$ \\
 & Linear Probe  & $92.48{\scriptstyle \pm 0.83}$ & $60.68{\scriptstyle \pm 5.61}$ & $95.26{\scriptstyle \pm 0.98}$ & $62.77{\scriptstyle \pm 6.06}$ & $74.03{\scriptstyle \pm 0.41}$ & $55.60{\scriptstyle \pm 15.95}$ & $74.18{\scriptstyle \pm 0.67}$ & $58.34{\scriptstyle \pm 7.13}$ & $62.13{\scriptstyle \pm 0.21}$ & $36.80{\scriptstyle \pm 4.82}$ & $63.76{\scriptstyle \pm 0.34}$ & $54.28{\scriptstyle \pm 3.55}$ \\
 & NCC & $92.48{\scriptstyle \pm 0.90}$ & $35.53{\scriptstyle \pm 6.51}$ & $95.34{\scriptstyle \pm 0.98}$ & $40.04{\scriptstyle \pm 7.86}$ & $73.87{\scriptstyle \pm 0.39}$ & $47.00{\scriptstyle \pm 15.41}$ & $74.12{\scriptstyle \pm 0.68}$ & $42.82{\scriptstyle \pm 6.21}$ & $61.66{\scriptstyle \pm 0.25}$ & $34.40{\scriptstyle \pm 6.07}$ & $63.28{\scriptstyle \pm 0.35}$ & $50.42{\scriptstyle \pm 3.23}$ \\

 \midrule
\multirow{3}{*}{UNSIR with CMF}& Output & $91.79{\scriptstyle \pm 0.92}$ & $12.91{\scriptstyle \pm 7.43}$ & $93.56{\scriptstyle \pm 1.02}$ & $11.51{\scriptstyle \pm 3.00}$ & $72.72{\scriptstyle \pm 0.15}$ & $21.00{\scriptstyle \pm 11.98}$ & $72.61{\scriptstyle \pm 1.10}$ & $9.16{\scriptstyle \pm 3.23}$ & $60.81{\scriptstyle \pm 0.22}$ & $14.00{\scriptstyle \pm 8.25}$ & $61.34{\scriptstyle \pm 0.54}$ & $14.44{\scriptstyle \pm 0.90}$ \\
 & Linear Probe & $91.63{\scriptstyle \pm 0.90}$ & $31.16{\scriptstyle \pm 6.35}$ & $93.01{\scriptstyle \pm 1.17}$ & $28.65{\scriptstyle \pm 3.13}$ & $72.91{\scriptstyle \pm 0.16}$ & $35.20{\scriptstyle \pm 17.25}$ & $72.51{\scriptstyle \pm 0.87}$ & $20.98{\scriptstyle \pm 5.73}$ & $60.96{\scriptstyle \pm 0.21}$ & $26.80{\scriptstyle \pm 3.35}$ & $61.27{\scriptstyle \pm 0.36}$ & $23.02{\scriptstyle \pm 0.59}$ \\
 & NCC & $91.87{\scriptstyle \pm 0.91}$ & $9.29{\scriptstyle \pm 4.22}$ & $93.79{\scriptstyle \pm 1.03}$ & $8.16{\scriptstyle \pm 2.32}$ & $72.67{\scriptstyle \pm 0.16}$ & $20.20{\scriptstyle \pm 10.57}$ & $72.48{\scriptstyle \pm 1.17}$ & $9.94{\scriptstyle \pm 3.48}$ & $60.63{\scriptstyle \pm 0.24}$ & $26.00{\scriptstyle \pm 3.16}$ & $60.80{\scriptstyle \pm 0.46}$ & $16.72{\scriptstyle \pm 1.44}$ \\
\bottomrule
\end{tabular}
}
\end{table*}
\endgroup

\begingroup
\setlength{\tabcolsep}{4pt}
\small

\begin{table*}[t]
\centering
\caption{
Evaluation of machine unlearning (MU) methods on \textbf{ViT-S/16} across CIFAR-10, CIFAR-100, and Tiny-ImageNet.
We report mean accuracy ($\pm$ standard deviation) on \textbf{retain} and \textbf{forget} subsets under different unlearning bucket sizes (single vs.\ multi-class).
Results are shown for three evaluation heads: the full model classifier, a linear probe on frozen features (LP), and the Nearest Class-Center (NC) classifier.
All experiments are based on models pre-trained on \textbf{ImageNet}.
For ViT-S/16, the pre-trained model is obtained by fine-tuning the ImageNet-pre-trained backbone on the full target dataset.
\emph{Retrain} refers to a model obtained by fine-tuning the same pre-trained backbone using only the retain subset, serving as an oracle baseline that fully removes the target data.
All unlearning methods are likewise initialized from the same pre-trained model.
All values are averaged over multiple runs..
}
\label{tab:Vit_original_linearprobe_with_var}
\vspace{-.15in}
\resizebox{1.0\linewidth}{!}{
\begin{tabular}{l l cc cc cc cc cc cc}
\toprule
\multirow{3}{*}{\textbf{Method}}
  & \multirow{3}{*}{\textbf{Accuracy}}
  & \multicolumn{4}{c}{\textbf{CIFAR-10}}
  & \multicolumn{4}{c}{\textbf{CIFAR-100}}
  & \multicolumn{4}{c}{\textbf{Tiny-ImageNet}}\\
\cmidrule(lr){3-6}\cmidrule(lr){7-10}\cmidrule(lr){11-14}
  &  & \multicolumn{2}{c}{1} & \multicolumn{2}{c}{3}
     & \multicolumn{2}{c}{1} & \multicolumn{2}{c}{10}
     & \multicolumn{2}{c}{1} & \multicolumn{2}{c}{20}\\
\cmidrule(lr){3-4}\cmidrule(lr){5-6}\cmidrule(lr){7-8}\cmidrule(lr){9-10}\cmidrule(lr){11-12}\cmidrule(lr){13-14}
  &  & Retain & Forget & Retain & Forget & Retain & Forget & Retain & Forget & Retain & Forget & Retain & Forget \\
\midrule
\multirow{3}{*}{Original} & Output & $98.20{\scriptstyle \pm 0.15}$ & $98.20{\scriptstyle \pm 1.39}$ & $98.16{\scriptstyle \pm 0.44}$ & $98.29{\scriptstyle \pm 1.02}$ & $90.09{\scriptstyle \pm 0.06}$ & $92.00{\scriptstyle \pm 5.87}$ & $90.09{\scriptstyle \pm 0.46}$ & $90.32{\scriptstyle \pm 4.18}$ & $84.90{\scriptstyle \pm 0.05}$ & $79.20{\scriptstyle \pm 10.83}$ & $84.81{\scriptstyle \pm 0.31}$ & $85.44{\scriptstyle \pm 2.82}$ \\
 & Linear Probe & $98.17{\scriptstyle \pm 0.14}$ & $98.17{\scriptstyle \pm 1.26}$ & $98.15{\scriptstyle \pm 0.36}$ & $98.22{\scriptstyle \pm 0.85}$ & $89.93{\scriptstyle \pm 0.05}$ & $92.20{\scriptstyle \pm 5.07}$ & $89.94{\scriptstyle \pm 0.43}$ & $90.08{\scriptstyle \pm 3.83}$ & $84.97{\scriptstyle \pm 0.04}$ & $81.20{\scriptstyle \pm 7.82}$ & $84.87{\scriptstyle \pm 0.26}$ & $85.68{\scriptstyle \pm 2.31}$ \\
 & NCC & $98.06{\scriptstyle \pm 0.20}$ & $98.06{\scriptstyle \pm 1.76}$ & $98.02{\scriptstyle \pm 0.50}$ & $98.15{\scriptstyle \pm 1.16}$ & $89.62{\scriptstyle \pm 0.05}$ & $93.00{\scriptstyle \pm 5.10}$ & $89.67{\scriptstyle \pm 0.47}$ & $89.58{\scriptstyle \pm 4.01}$ & $83.62{\scriptstyle \pm 0.03}$ & $80.40{\scriptstyle \pm 6.84}$ & $83.49{\scriptstyle \pm 0.31}$ & $84.70{\scriptstyle \pm 2.83}$ \\

 \midrule
\multirow{3}{*}{Retain-only Retrain} & Output & $98.46{\scriptstyle \pm 0.29}$ & $0.00{\scriptstyle \pm 0.00}$ & $98.61{\scriptstyle \pm 0.51}$ & $0.00{\scriptstyle \pm 0.00}$ & $89.95{\scriptstyle \pm 0.05}$ & $0.00{\scriptstyle \pm 0.00}$ & $90.10{\scriptstyle \pm 0.46}$ & $0.00{\scriptstyle \pm 0.00}$ & $83.85{\scriptstyle \pm 0.14}$ & $0.00{\scriptstyle \pm 0.00}$ & $84.14{\scriptstyle \pm 0.67}$ & $0.00{\scriptstyle \pm 0.00}$ \\
 & Linear Probe & $98.28{\scriptstyle \pm 0.08}$ & $97.71{\scriptstyle \pm 1.24}$ & $98.17{\scriptstyle \pm 0.36}$ & $97.35{\scriptstyle \pm 1.32}$ & $90.32{\scriptstyle \pm 0.11}$ & $90.80{\scriptstyle \pm 7.09}$ & $90.13{\scriptstyle \pm 0.42}$ & $87.82{\scriptstyle \pm 4.67}$ & $84.47{\scriptstyle \pm 0.13}$ & $77.60{\scriptstyle \pm 5.55}$ & $84.45{\scriptstyle \pm 0.38}$ & $81.44{\scriptstyle \pm 2.79}$ \\
 & NCC & $98.12{\scriptstyle \pm 0.17}$ & $94.24{\scriptstyle \pm 3.73}$ & $97.98{\scriptstyle \pm 0.43}$ & $94.51{\scriptstyle \pm 1.95}$ & $89.24{\scriptstyle \pm 0.09}$ & $87.80{\scriptstyle \pm 5.97}$ & $89.31{\scriptstyle \pm 0.36}$ & $81.88{\scriptstyle \pm 5.19}$ & $82.93{\scriptstyle \pm 0.05}$ & $74.80{\scriptstyle \pm 5.76}$ & $82.74{\scriptstyle \pm 0.30}$ & $78.64{\scriptstyle \pm 3.19}$ \\

\midrule
\multirow{3}{*}{Random-label} & Output & $98.50{\scriptstyle \pm 0.25}$ & $0.00{\scriptstyle \pm 0.00}$ & $98.69{\scriptstyle \pm 0.43}$ & $0.00{\scriptstyle \pm 0.00}$ & $90.59{\scriptstyle \pm 0.10}$ & $0.00{\scriptstyle \pm 0.00}$ & $90.79{\scriptstyle \pm 0.42}$ & $0.10{\scriptstyle \pm 0.17}$ & $85.87{\scriptstyle \pm 0.07}$ & $7.20{\scriptstyle \pm 3.35}$ & $86.23{\scriptstyle \pm 0.55}$ & $16.24{\scriptstyle \pm 4.51}$ \\
 & Linear Probe & $98.35{\scriptstyle \pm 0.10}$ & $98.64{\scriptstyle \pm 1.08}$ & $98.15{\scriptstyle \pm 0.47}$ & $98.26{\scriptstyle \pm 1.56}$ & $90.47{\scriptstyle \pm 0.05}$ & $94.00{\scriptstyle \pm 5.15}$ & $90.26{\scriptstyle \pm 0.37}$ & $88.32{\scriptstyle \pm 5.18}$ & $85.69{\scriptstyle \pm 0.07}$ & $82.80{\scriptstyle \pm 5.40}$ & $85.57{\scriptstyle \pm 0.23}$ & $85.10{\scriptstyle \pm 2.47}$ \\
 & NCC & $97.50{\scriptstyle \pm 0.24}$ & $99.57{\scriptstyle \pm 0.46}$ & $96.46{\scriptstyle \pm 0.70}$ & $97.84{\scriptstyle \pm 3.37}$ & $90.01{\scriptstyle \pm 0.09}$ & $96.00{\scriptstyle \pm 4.64}$ & $89.56{\scriptstyle \pm 0.62}$ & $82.70{\scriptstyle \pm 6.70}$ & $84.22{\scriptstyle \pm 0.09}$ & $82.40{\scriptstyle \pm 6.84}$ & $84.42{\scriptstyle \pm 0.35}$ & $83.92{\scriptstyle \pm 3.07}$ \\
\midrule
\multirow{3}{*}{Salun} & Output & $98.46{\scriptstyle \pm 0.25}$ & $0.00{\scriptstyle \pm 0.00}$ & $98.68{\scriptstyle \pm 0.46}$ & $0.00{\scriptstyle \pm 0.00}$ & $90.58{\scriptstyle \pm 0.13}$ & $0.20{\scriptstyle \pm 0.45}$ & $90.64{\scriptstyle \pm 0.46}$ & $0.18{\scriptstyle \pm 0.16}$ & $85.64{\scriptstyle \pm 0.06}$ & $7.60{\scriptstyle \pm 5.18}$ & $85.99{\scriptstyle \pm 0.71}$ & $15.72{\scriptstyle \pm 4.48}$ \\
 & Linear Probe & $98.33{\scriptstyle \pm 0.18}$ & $98.59{\scriptstyle \pm 1.16}$ & $98.13{\scriptstyle \pm 0.35}$ & $98.26{\scriptstyle \pm 1.59}$ & $90.36{\scriptstyle \pm 0.19}$ & $93.60{\scriptstyle \pm 5.18}$ & $90.22{\scriptstyle \pm 0.51}$ & $88.24{\scriptstyle \pm 4.93}$ & $85.50{\scriptstyle \pm 0.09}$ & $81.60{\scriptstyle \pm 7.67}$ & $85.33{\scriptstyle \pm 0.32}$ & $84.74{\scriptstyle \pm 2.58}$ \\
 & NCC & $97.51{\scriptstyle \pm 0.19}$ & $99.51{\scriptstyle \pm 0.49}$ & $96.61{\scriptstyle \pm 0.46}$ & $97.85{\scriptstyle \pm 3.50}$ & $89.96{\scriptstyle \pm 0.11}$ & $95.60{\scriptstyle \pm 4.88}$ & $89.33{\scriptstyle \pm 0.62}$ & $81.68{\scriptstyle \pm 4.38}$ & $84.16{\scriptstyle \pm 0.11}$ & $82.00{\scriptstyle \pm 8.72}$ & $84.11{\scriptstyle \pm 0.32}$ & $83.46{\scriptstyle \pm 3.06}$ \\

\midrule
\multirow{3}{*}{NegGrad+} & Output & $98.00{\scriptstyle \pm 0.33}$ & $0.00{\scriptstyle \pm 0.00}$ & $98.21{\scriptstyle \pm 0.56}$ & $6.63{\scriptstyle \pm 14.82}$ & $89.88{\scriptstyle \pm 0.26}$ & $0.00{\scriptstyle \pm 0.00}$ & $89.85{\scriptstyle \pm 0.55}$ & $6.14{\scriptstyle \pm 5.61}$ & $85.12{\scriptstyle \pm 0.26}$ & $0.00{\scriptstyle \pm 0.00}$ & $85.48{\scriptstyle \pm 0.73}$ & $0.52{\scriptstyle \pm 0.69}$ \\
 & Linear Probe & $98.03{\scriptstyle \pm 0.16}$ & $96.15{\scriptstyle \pm 2.47}$ & $97.57{\scriptstyle \pm 0.67}$ & $95.76{\scriptstyle \pm 1.13}$ & $90.24{\scriptstyle \pm 0.19}$ & $87.20{\scriptstyle \pm 7.43}$ & $89.96{\scriptstyle \pm 0.55}$ & $82.98{\scriptstyle \pm 4.66}$ & $84.95{\scriptstyle \pm 0.18}$ & $75.20{\scriptstyle \pm 5.93}$ & $84.74{\scriptstyle \pm 0.29}$ & $67.50{\scriptstyle \pm 4.21}$ \\
 & NCC & $94.88{\scriptstyle \pm 1.60}$ & $87.05{\scriptstyle \pm 5.49}$ & $93.86{\scriptstyle \pm 2.70}$ & $73.38{\scriptstyle \pm 4.37}$ & $89.00{\scriptstyle \pm 0.32}$ & $86.80{\scriptstyle \pm 6.26}$ & $89.20{\scriptstyle \pm 0.73}$ & $43.40{\scriptstyle \pm 5.94}$ & $83.29{\scriptstyle \pm 0.29}$ & $84.40{\scriptstyle \pm 5.55}$ & $82.94{\scriptstyle \pm 0.30}$ & $57.16{\scriptstyle \pm 7.44}$ \\

\bottomrule
\end{tabular}
}
\end{table*}

\endgroup

\begingroup
\setlength{\tabcolsep}{4pt}
\small
\begin{table*}[t]
\centering
\caption{
Evaluation of \textbf{CMF-based machine unlearning} methods on the \textbf{ViT-S/16} model across CIFAR-10, CIFAR-100, and Tiny-ImageNet.
Results report mean accuracy on retain and forget subsets under different unlearning bucket sizes.
These results demonstrate that CMF-based unlearning effectively removes information not only at the output level but also in the \textbf{feature representations} of ImageNet-pretrained vision transformers.
}
\label{tab:eval_vit_s_16_with_var}
\vspace{-0.12in}
\resizebox{1.0\linewidth}{!}{
\begin{tabular}{l l cccc cccc cccc}
\toprule
\multirow{3}{*}{\textbf{Method}} & \multirow{3}{*}{\textbf{Accuracy}} & \multicolumn{4}{c}{\textbf{CIFAR-10}} & \multicolumn{4}{c}{\textbf{CIFAR-100}} & \multicolumn{4}{c}{\textbf{Tiny-ImageNet}} \\
\cmidrule(lr){3-6}\cmidrule(lr){7-10}\cmidrule(lr){11-14}
 &  & \multicolumn{2}{c}{\textbf{1}} & \multicolumn{2}{c}{\textbf{3}} & \multicolumn{2}{c}{\textbf{1}} & \multicolumn{2}{c}{\textbf{10}} & \multicolumn{2}{c}{\textbf{1}} & \multicolumn{2}{c}{\textbf{20}} \\
\cmidrule(lr){3-4}\cmidrule(lr){5-6}\cmidrule(lr){7-8}\cmidrule(lr){9-10}\cmidrule(lr){11-12}\cmidrule(lr){13-14}
 &  & \textbf{Retain} & \textbf{Forget} & \textbf{Retain} & \textbf{Forget} & \textbf{Retain} & \textbf{Forget} & \textbf{Retain} & \textbf{Forget} & \textbf{Retain} & \textbf{Forget} & \textbf{Retain} & \textbf{Forget} \\
\midrule
\multirow{3}{*}{Original} & Output & $98.20{\scriptstyle \pm 0.15}$ & $98.20{\scriptstyle \pm 1.39}$ & $98.16{\scriptstyle \pm 0.44}$ & $98.29{\scriptstyle \pm 1.02}$ & $90.09{\scriptstyle \pm 0.06}$ & $92.00{\scriptstyle \pm 5.87}$ & $90.09{\scriptstyle \pm 0.46}$ & $90.32{\scriptstyle \pm 4.18}$ & $84.90{\scriptstyle \pm 0.05}$ & $79.20{\scriptstyle \pm 10.83}$ & $84.81{\scriptstyle \pm 0.31}$ & $85.44{\scriptstyle \pm 2.82}$ \\
 & Linear Probe & $98.17{\scriptstyle \pm 0.14}$ & $98.17{\scriptstyle \pm 1.26}$ & $98.15{\scriptstyle \pm 0.36}$ & $98.22{\scriptstyle \pm 0.85}$ & $89.93{\scriptstyle \pm 0.05}$ & $92.20{\scriptstyle \pm 5.07}$ & $89.94{\scriptstyle \pm 0.43}$ & $90.08{\scriptstyle \pm 3.83}$ & $84.97{\scriptstyle \pm 0.04}$ & $81.20{\scriptstyle \pm 7.82}$ & $84.87{\scriptstyle \pm 0.26}$ & $85.68{\scriptstyle \pm 2.31}$ \\
 & NCC & $98.06{\scriptstyle \pm 0.20}$ & $98.06{\scriptstyle \pm 1.76}$ & $98.02{\scriptstyle \pm 0.50}$ & $98.15{\scriptstyle \pm 1.16}$ & $89.62{\scriptstyle \pm 0.05}$ & $93.00{\scriptstyle \pm 5.10}$ & $89.67{\scriptstyle \pm 0.47}$ & $89.58{\scriptstyle \pm 4.01}$ & $83.62{\scriptstyle \pm 0.03}$ & $80.40{\scriptstyle \pm 6.84}$ & $83.49{\scriptstyle \pm 0.31}$ & $84.70{\scriptstyle \pm 2.83}$ \\
 
 \midrule
\multirow{3}{*}{Retain-only Retrain} & Output & $98.46{\scriptstyle \pm 0.29}$ & $0.00{\scriptstyle \pm 0.00}$ & $98.61{\scriptstyle \pm 0.51}$ & $0.00{\scriptstyle \pm 0.00}$ & $89.95{\scriptstyle \pm 0.05}$ & $0.00{\scriptstyle \pm 0.00}$ & $90.10{\scriptstyle \pm 0.46}$ & $0.00{\scriptstyle \pm 0.00}$ & $83.85{\scriptstyle \pm 0.14}$ & $0.00{\scriptstyle \pm 0.00}$ & $84.14{\scriptstyle \pm 0.67}$ & $0.00{\scriptstyle \pm 0.00}$ \\
 & Linear Probe & $98.28{\scriptstyle \pm 0.08}$ & $97.71{\scriptstyle \pm 1.24}$ & $98.17{\scriptstyle \pm 0.36}$ & $97.35{\scriptstyle \pm 1.32}$ & $90.32{\scriptstyle \pm 0.11}$ & $90.80{\scriptstyle \pm 7.09}$ & $90.13{\scriptstyle \pm 0.42}$ & $87.82{\scriptstyle \pm 4.67}$ & $84.47{\scriptstyle \pm 0.13}$ & $77.60{\scriptstyle \pm 5.55}$ & $84.45{\scriptstyle \pm 0.38}$ & $81.44{\scriptstyle \pm 2.79}$ \\
 & NCC & $98.12{\scriptstyle \pm 0.17}$ & $94.24{\scriptstyle \pm 3.73}$ & $97.98{\scriptstyle \pm 0.43}$ & $94.51{\scriptstyle \pm 1.95}$ & $89.24{\scriptstyle \pm 0.09}$ & $87.80{\scriptstyle \pm 5.97}$ & $89.31{\scriptstyle \pm 0.36}$ & $81.88{\scriptstyle \pm 5.19}$ & $82.93{\scriptstyle \pm 0.05}$ & $74.80{\scriptstyle \pm 5.76}$ & $82.74{\scriptstyle \pm 0.30}$ & $78.64{\scriptstyle \pm 3.19}$ \\

\midrule
\multirow{3}{*}{Random-label with CMF} & Output & $97.38{\scriptstyle \pm 0.30}$ & $52.77{\scriptstyle \pm 15.71}$ & $96.78{\scriptstyle \pm 1.23}$ & $46.83{\scriptstyle \pm 8.54}$ & $86.99{\scriptstyle \pm 0.18}$ & $55.20{\scriptstyle \pm 14.82}$ & $87.81{\scriptstyle \pm 0.78}$ & $53.06{\scriptstyle \pm 14.14}$ & $77.60{\scriptstyle \pm 0.10}$ & $61.60{\scriptstyle \pm 8.41}$ & $78.12{\scriptstyle \pm 0.61}$ & $62.56{\scriptstyle \pm 2.43}$ \\
 & Linear Probe & $97.20{\scriptstyle \pm 0.33}$ & $72.30{\scriptstyle \pm 9.32}$ & $96.19{\scriptstyle \pm 1.14}$ & $66.69{\scriptstyle \pm 8.22}$ &  $84.79{\scriptstyle \pm 0.18}$ & $24.60{\scriptstyle \pm 14.77}$ & $86.01{\scriptstyle \pm 0.97}$ & $41.10{\scriptstyle \pm 16.54}$ & $80.20{\scriptstyle \pm 0.09}$ & $68.40{\scriptstyle \pm 7.92}$ & $80.57{\scriptstyle \pm 0.52}$ & $71.04{\scriptstyle \pm 3.19}$ \\
 & NCC & $97.35{\scriptstyle \pm 0.29}$ & $55.82{\scriptstyle \pm 14.25}$ & $96.66{\scriptstyle \pm 1.20}$ & $49.97{\scriptstyle \pm 9.02}$ & $86.97{\scriptstyle \pm 0.20}$ & $55.20{\scriptstyle \pm 11.90}$ & $87.79{\scriptstyle \pm 0.79}$ & $52.54{\scriptstyle \pm 14.90}$ & $77.09{\scriptstyle \pm 0.11}$ & $61.20{\scriptstyle \pm 8.44}$ & $77.55{\scriptstyle \pm 0.59}$ & $60.48{\scriptstyle \pm 3.34}$ \\

\midrule
\multirow{3}{*}{Salun with CMF} & Output & $97.30{\scriptstyle \pm 0.38}$ & $53.69{\scriptstyle \pm 15.79}$ & $96.94{\scriptstyle \pm 1.20}$ & $49.31{\scriptstyle \pm 10.28}$& $86.20{\scriptstyle \pm 0.43}$ & $53.20{\scriptstyle \pm 24.73}$ & $87.42{\scriptstyle \pm 0.69}$ & $46.40{\scriptstyle \pm 12.73}$ & $77.91{\scriptstyle \pm 0.07}$ & $68.00{\scriptstyle \pm 8.00}$ & $78.34{\scriptstyle \pm 0.76}$ & $68.66{\scriptstyle \pm 3.22}$ \\
 & Linear Probe & $97.11{\scriptstyle \pm 0.37}$ & $71.98{\scriptstyle \pm 9.51}$ & $96.23{\scriptstyle \pm 1.28}$ & $68.07{\scriptstyle \pm 8.89}$ & $83.43{\scriptstyle \pm 0.59}$ & $23.20{\scriptstyle \pm 18.21}$ & $85.37{\scriptstyle \pm 0.89}$ & $34.70{\scriptstyle \pm 14.59}$ & $81.07{\scriptstyle \pm 0.11}$ & $74.40{\scriptstyle \pm 8.88}$ & $81.36{\scriptstyle \pm 0.32}$ & $75.24{\scriptstyle \pm 2.45}$ \\
 & NCC & $97.26{\scriptstyle \pm 0.35}$ & $56.14{\scriptstyle \pm 14.52}$ & $96.81{\scriptstyle \pm 1.23}$ & $52.30{\scriptstyle \pm 10.73}$ & $86.20{\scriptstyle \pm 0.46}$ & $51.20{\scriptstyle \pm 22.88}$ & $87.39{\scriptstyle \pm 0.70}$ & $46.34{\scriptstyle \pm 12.84}$& $77.61{\scriptstyle \pm 0.08}$ & $64.00{\scriptstyle \pm 8.49}$ & $77.84{\scriptstyle \pm 0.77}$ & $67.62{\scriptstyle \pm 3.17}$ \\

 \midrule
\multirow{3}{*}{NegGrad+ with CMF} & Output & $92.94{\scriptstyle \pm 2.02}$ & $41.42{\scriptstyle \pm 14.64}$ & $93.41{\scriptstyle \pm 1.23}$ & $48.50{\scriptstyle \pm 6.68}$ & $85.48{\scriptstyle \pm 1.45}$ & $55.80{\scriptstyle \pm 13.18}$ & $83.91{\scriptstyle \pm 2.25}$ & $47.08{\scriptstyle \pm 5.58}$ & $77.81{\scriptstyle \pm 1.57}$ & $17.20{\scriptstyle \pm 8.44}$ & $65.38{\scriptstyle \pm 8.49}$ & $31.24{\scriptstyle \pm 4.90}$ \\
 & Linear Probe & $93.76{\scriptstyle \pm 1.60}$ & $55.61{\scriptstyle \pm 10.32}$ & $93.68{\scriptstyle \pm 1.02}$ & $56.75{\scriptstyle \pm 8.04}$ & $83.47{\scriptstyle \pm 1.72}$ & $24.40{\scriptstyle \pm 15.79}$ & $81.19{\scriptstyle \pm 2.46}$ & $22.30{\scriptstyle \pm 4.31}$ & $81.06{\scriptstyle \pm 1.11}$ & $34.40{\scriptstyle \pm 12.36}$ & $69.65{\scriptstyle \pm 8.88}$ & $33.98{\scriptstyle \pm 9.05}$ \\
 & NCC & $93.04{\scriptstyle \pm 1.92}$ & $43.89{\scriptstyle \pm 12.64}$ & $93.45{\scriptstyle \pm 1.16}$ & $48.08{\scriptstyle \pm 5.34}$ & $85.53{\scriptstyle \pm 1.47}$ & $58.40{\scriptstyle \pm 13.32}$ & $83.97{\scriptstyle \pm 2.30}$ & $45.28{\scriptstyle \pm 5.97}$ & $77.10{\scriptstyle \pm 1.82}$ & $12.00{\scriptstyle \pm 3.74}$ & $63.65{\scriptstyle \pm 8.81}$ & $27.58{\scriptstyle \pm 4.94}$ \\

\bottomrule
\end{tabular}
}
\end{table*}
\endgroup

\end{document}